\documentclass[twoside,11pt]{article}

\usepackage{jmlr2e}
\usepackage{booktabs}
\usepackage{algorithm, algorithmic}
\usepackage{color}

\newcommand{\blue}[1]{\begin{color}{blue}#1\end{color}}

\def\R{\mathbb{R}}
\def\norm#1{\|#1\|}
\def\inprod#1#2{\langle #1,\,#2\rangle}
\def\Prox{\mbox{Prox}}
\def\bx{\mathbf{x}}
\def\by{\mathbf{y}}
\def\bv{\mathbf{v}}
\def\ba{\mathbf{a}}
\def\bz{\mathbf{z}}
\def\bp{\mathbf{p}}
\def\R{\mathbb{R}}
\def\alp{\alpha}
\def\lam{\lambda}
\def\gam{\gamma}
\def\mb#1{\mathbf{#1}}
\def\cB{{\cal B}} \def\cE{{\cal E}} \def\cP{{\cal P}}

\ShortHeadings{Convex Clustering: Model, Theoretical Guarantee and Efficient Algorithm}{Sun, Toh and Yuan}
\firstpageno{1}

\begin{document}

\title{Convex Clustering: Model, Theoretical Guarantee and Efficient Algorithm}

\author{\name Defeng Sun \email defeng.sun@polyu.edu.hk \\
       \addr Department of Applied Mathematics\\
       \blue{The Hong Kong Polytechnic University}\\
       Hong Kong
       \AND
       \name Kim-Chuan Toh \email mattohkc@nus.edu.sg \\
       \addr Department of Mathematics  and Institute of Operations Research and Analytics \\
       \blue{National University of Singapore}\\
       10 Lower Kent Ridge Road, Singapore
       119076
       \AND
       \name Yancheng Yuan \email yuanyancheng@u.nus.edu \\
       \addr Department of Mathematics\\
       \blue{National University of Singapore}\\
       10 Lower Kent Ridge Road, Singapore
       119076
   }

\editor{}

\maketitle

\begin{abstract}
Clustering is a fundamental problem in unsupervised learning. Popular methods like K-means, may suffer from
poor performance as they are prone to get stuck in its local minima. Recently, the sum-of-norms (SON) model (also known as the clustering path) 
has been proposed in \citet{pelckmans05}, \citet{Lindsten11} and \citet{Hocking11}. The perfect recovery properties of the convex clustering model with uniformly weighted all-pairwise-differences regularization have been proved by \cite{NIPS2014_5307} and \cite{pmlr-v70-panahi17a}. However,  no theoretical guarantee has been established for the
general weighted convex clustering model, where better empirical results have been observed. In the numerical optimization aspect, although algorithms like the alternating direction method of multipliers (ADMM) and the alternating minimization algorithm (AMA) have been proposed to solve the convex clustering model \citep{Chi15}, it still remains  very challenging to solve large-scale problems. In this paper, we establish sufficient conditions for the perfect recovery guarantee of the general weighted convex clustering model, which include and improve existing theoretical results as special cases. In addition, we develop a semismooth Newton based augmented Lagrangian method for solving large-scale convex clustering problems. Extensive numerical experiments on both simulated and real data demonstrate that our algorithm is  highly efficient and robust for solving large-scale problems. Moreover, the numerical results also show the superior performance and scalability of our algorithm comparing to the existing first-order methods. In particular, our algorithm is able to solve a convex clustering problem
with 200,000 points in $\R^3$ in about 6 minutes.
\end{abstract}


\section{Introduction}

Clustering is one of the most fundamental  problems in unsupervised learning. Traditional clustering models such as K-means clustering, hierarchical clustering may suffer from poor performance
because of the non-convexity of the models and the difficulties in finding global optimal solutions for such models. The clustering results are generally highly dependent on the initializations and the results could differ significantly with different initializations.
Moreover, these clustering models require the prior knowledge about the number of clusters which is not available in many real applications. Therefore, in practice, K-means is typically tried with different cluster numbers and the user will then decide on a suitable value based on his judgment on which computed result agrees best with his domain knowledge. Obviously, such a process could make the clustering results subjective.

In order to overcome the above issues, a new clustering model has been proposed \citep{pelckmans05, Lindsten11, Hocking11} and demonstrated to be more robust compared to those traditional ones. Let $A \in \mathbb{R}^{d \times n} = [\mathbf{a}_1, \mathbf{a}_2, \cdots, \mathbf{a}_n]$ be a given data matrix with $n$ observations and $d$ features. The convex clustering model for these $n$ observations solves the following convex optimization problem:
\begin{eqnarray}\label{Eq: Origin_Model}
\min_{X \in \R^{d \times n}} \frac{1}{2} \sum_{i=1}^{n} \|\mathbf{x}_i - \mathbf{a}_i\|^2 + \gamma\sum_{i < j} \|\mathbf{x}_i - \mathbf{x}_j\|_p,
\end{eqnarray}
where  $\gamma>0$ is a tuning parameter, and $\norm{\cdot}_p$ denotes the $p$-norm.
Here and below, $\norm{\cdot}$ is used to denote the vector $2$-norm or the Frobenius norm of a matrix.
The $p$-norm  above with $p \geq 1$ ensures the convexity of the model.
Typically $p$ is chosen to be $1, 2,$ or $\infty$. 
After solving (\ref{Eq: Origin_Model}) and obtaining the optimal solution $X^* = [\mb{x}_1^*,\ldots,\mb{x}_n^*]$,
we assign
$\mathbf{a}_i$ and $\mathbf{a}_j$ to the same cluster if and only if $\mathbf{x}^*_i = \mathbf{x}^*_j$. In other words, $\mathbf{x}^*_i$ is the centroid for observation $\mathbf{a}_i$. (Here we used the word ``centroid'' to mean the
	approximate one associated with $\mathbf{a}_i$ but not the final centroid of the cluster to which
	$\mathbf{a}_i$ belongs to.)
The idea behind this model is that if two observations $\mathbf{a}_i$ and $\mathbf{a}_j$ belong to the same cluster, then their corresponding centroids $\mathbf{x}^*_i$ and $\mathbf{x}^*_j$ should be the same.  The first term in (\ref{Eq: Origin_Model}) is the fidelity term while the second term
is the regularization term to penalize the differences between different centroids so as to enforce the property that
centroids for observations in the same cluster should be identical.

The advantages of convex clustering lie mainly in two aspects. First, since the clustering model (\ref{Eq: Origin_Model}) is strongly convex, the optimal solution for a given positive $\gamma$ is unique and is more easily
obtainable than traditional clustering algorithms like K-means. Second, instead of requiring the prior knowledge of the cluster number, we can generate a clustering path via solving (\ref{Eq: Origin_Model}) for a sequence of positive values of $\gamma$.
To handle cluster recovery for large-scale data sets, various researchers,
e.g., \cite{pelckmans05,Lindsten11, Hocking11, NIPS2014_5307,  tan2015statistical, pmlr-v70-panahi17a}
have suggested the following weighted clustering model modified from (\ref{Eq: Origin_Model}):
\begin{equation}\label{Eq: Modified_Model}
\min_{X \in \R^{d \times n}} \frac{1}{2} \sum_{i=1}^{n} \|\mathbf{x}_i - \mathbf{a}_i\|^2 + \gamma\sum_{i < j} w_{ij}\|\mathbf{x}_i - \mathbf{x}_j\|_p,
\end{equation}
where $w_{ij} = w_{ji} \geq 0$ are given weights that are generally chosen based on the given input data $A$.
One can regard the original convex clustering model (\ref{Eq: Origin_Model}) as a special case if we take $w_{ij} = 1$ for all $i < j$. To make the computational cost cheaper
when evaluating the regularization term, one would generally put a non-zero weight only for a
pair of points which are nearby each other, and  a typical choice of the weights is
$$
w_{ij} = \left\{
\begin{array}{ll}
\exp(-\phi\|\mathbf{a}_i - \mathbf{a}_j\|^2)  & {\rm if} \; (i, j)\in \mathcal{E},\\[5pt]
0 & {\rm otherwise},
\end{array}
\right.
$$
where $\mathcal{E} = \cup_{i=1}^n \{(i, j) \mid  \mbox{$j$ is among $i$'s $k$-nearest}$ $\mbox{neighbors}, i< j\leq n\}$.

The advantages just mentioned and the success of the convex model (\ref{Eq: Origin_Model}) in recovering clusters in many examples with well selected values of $\gamma$ have motivated researchers to provide theoretical guarantees on the cluster recovery property of (\ref{Eq: Origin_Model}). The first theoretical result on cluster recovery
established in \citep{NIPS2014_5307} is valid for only two clusters. It showed that the model (\ref{Eq: Origin_Model}) can recover the two clusters perfectly if the data points are drawn from two cubes that well separated. \citet{tan2015statistical} analyzed the statistical properties of (\ref{Eq: Origin_Model}). Recently,  \citet{pmlr-v70-panahi17a} provided  theoretical recovery results in the general $k$ clusters case under relatively mild sufficient conditions, for the fully uniformly weighted convex model (\ref{Eq: Origin_Model}).

In the practical aspect, various researchers have observed that better empirical performance can be
achieved by (\ref{Eq: Modified_Model}) with well chosen weights when comparing to the original  model (\ref{Eq: Origin_Model}) \citep{Hocking11, Lindsten11, Chi15}. However, to the best of our knowledge, no theoretical recovery guarantee has been established for the general weighted convex clustering model (\ref{Eq: Modified_Model}). In this paper, we will propose mild sufficient conditions for (\ref{Eq: Modified_Model}) to attain perfect recovery guarantee, which also include and improve the theoretical results in \citep{NIPS2014_5307,pmlr-v70-panahi17a} as special cases. Our theoretical results thus definitively strengthen the theoretical foundation of convex clustering model.
As expected, the conditions provided in the theoretical analysis are usually not checkable before one find the right clusters and thus the range of parameter values for $\gamma$ to achieve perfect recovery is unknown a priori. In practice,
this difficulty is mitigated by choosing a sequence of values of $\gamma$ to generate a clustering path.

The challenges for the convex model to obtain meaningful cluster recovery is then to solve it efficiently for a range of values of $\gamma$.  \citet{Lindsten11} used the off-the-shelf solver, CVX, to generate the solution path. However, \citet{Hocking11} realized that CVX is competitive only
for small-scale problems and it does not scale well when the number of data points increases. Thus
the paper introduced three algorithms based on the subgradient methods
for different regularizers corresponding to $p=1,2,\infty$. Recently, some new algorithms have been proposed to solve this problem.
\citet{Chi15} adapted the ADMM and AMA to solve (\ref{Eq: Origin_Model}). However,
as we will see in our numerical experiments, both algorithms may still encounter scalability issues, albeit less severe than CVX. Furthermore, the efficiency of these two algorithms is sensitive to the parameter value $\gamma$.
This is not a favorable property since we need to
solve (\ref{Eq: Origin_Model}) with $\gamma$ in a relative large range to generate the clustering path. In
\citet{pmlr-v70-panahi17a},  the authors proposed a stochastic splitting algorithm for (\ref{Eq: Origin_Model})
in an attempt to resolve the aforementioned scalability issues.
Although this stochastic approach scales well with the problem scale ($n$ in (\ref{Eq: Origin_Model})), the convergence rate shown in \cite{pmlr-v70-panahi17a} is rather weak
in that it requires at least $l\geq n^4/\varepsilon$ iterations to generate a solution $X^l$ such that
$\norm{X^l-X^*}^2\leq \varepsilon$ is satisfied with high probability.
Moreover, because the error estimate is given in the sense of high probability, it is difficult to design
an appropriate stopping condition for the algorithm in practice.

As the readers may observe, all the existing algorithms are purely first-order methods
that do not use any second-order information underlying the convex clustering model. In contrast, here we design and analyse a deterministic second-order algorithm,
the semismooth Newton based augmented Lagrangian method, to solve the convex clustering model.
Our algorithm is motivated by the recent work
\cite{li2016highly} in which the authors have proposed a semismooth Newton augmented Lagrangian method (ALM) to solve Lasso and fused Lasso problems, and the algorithm is demonstrated to be highly efficient for solving large, or even huge scale problems accurately. We are thus inspired
to adapt this ALM framework for solving the convex clustering model (\ref{Eq: Modified_Model}) in this paper.

Next we present a short summary of our main contributions in this paper.
\begin{itemize}
	\item[1.] We prove the perfect recovery guarantee of  the
	 general weighted convex clustering model (\ref{Eq: Modified_Model})
	 under mild sufficient conditions.
	Our results are not only applicable to the more practical  weighted convex model but also
	improve the existing results when specialized to the fully uniformly weighted model (\ref{Eq: Origin_Model}).
	Moreover, our bounds for the tuning parameter $\gamma$  are given explicitly in terms of
	the data points and their corresponding pairwise weights in the regularization term.
	
	\item[2.] We propose a highly efficient and scalable algorithm, called the semismooth Newton based augmented Lagrangian method, to solve the convex clustering model, which is not only proven to be theoretically efficient but it is also demonstrated to be practically highly efficient and robust.
\end{itemize}

The remaining parts of this paper are organized as follows. We will summarize some related work in section \ref{sect:related}. In section \ref{sect:Preli}, we will introduce some preliminaries and notation  which will be used in this paper. Theoretical results on the perfect recovery properties of the convex clustering  model will be presented in section 4. In section 5, we will introduce a highly efficient and robust optimization algorithm for solving the convex clustering model. After that, we will conduct numerical experiments
to verify the theoretical results and evaluate the performance of our
algorithm in section 6. Finally, we conclude the paper in section 7.

\section{Related Work Based on Semidefinite Programming}\label{sect:related}


In addition to the papers \citep{pelckmans05,Lindsten11, Hocking11, NIPS2014_5307,  tan2015statistical, pmlr-v70-panahi17a, chi2018provable} on the convex models (\ref{Eq: Origin_Model}) and (\ref{Eq: Modified_Model}), other convex models have been proposed to deal with the non-convexity of the K-means clustering model.
One such model is the convex relaxation of the K-means model via semidefinite programming (SDP) \citep{peng2007approximating, awasthi2015relax, mixon2016clustering}.

For a given data matrix $A \in \mathbb{R}^{d \times n} = [\ba_1, \ba_2, \dots, \ba_n]$, the classical K-means model solves the following non-convex optimization problem
\begin{equation}\label{eq: K-means}
\begin{array}{ll}
\min & \sum_{t=1}^k \sum_{i \in I_t} \|\ba_i - \frac{1}{|I_t|}\sum_{j \in I_t}\ba_j\|^2
\\[5pt]
{\rm s.t.} & \mbox{$ I_1, \ldots, I_k$ is a partition of   $\{1, 2, \dots, n\}$}.
\end{array}
\end{equation}
Now, if we define the $n \times n$ matrix $D$ by $D_{ij} = \|\ba_i - \ba_j\|^2$, then by taking
$$X := \sum_{t=1}^k \frac{1}{|I_t|}\mathbf{1}_{I_t}\mathbf{1}_{I_t}^T, $$
where $\mathbf{1}_{I_t}\in \R^n$ is the indicator vector of the index set $I_t$. We can express the objective function in (\ref{eq: K-means}) as $\frac{1}{2}{\rm Tr}(DX)$. Based on this, \cite{peng2007approximating} proposed the following SDP relaxation of the K-means model
\begin{equation}\label{eq: SDP-Relaxation-kmeans}
\begin{array}{ll}
\min  \Big\{ {\rm Tr}(DX)  \,\mid\, {\rm Tr}(X) = k, \; X\mathbf{e} = \mathbf{e}, \; X \geq 0, \; X \in \mathbb{S}_n^+
\Big\},
\end{array}
\end{equation}
where $X \geq 0$ means that all the elements in $X$ are nonnegative,
$\mathbb{S}^n_+$ is the cone of $n \times n$ symmetric and positive semidefinite matrices, and  $\mathbf{e} \in \R^n$ is the column vector of all ones.

Recently, \cite{mixon2016clustering} proved that the K-means SDP relaxation approach
can achieve perfect cluster recovery with high probability when the data $A$ is sampled from the
stochastic unit-ball model, provided that the cluster centriods $\{ \ba^{(1)},\ldots,\ba^{(k)}\}$ satisfy the condition that
$\min\{\norm{\ba^{(\alp)} - \ba^{(\beta)}} \mid 1\leq \alp < \beta \leq k\} > 2\sqrt{2}(1+1/\sqrt{d}).$
However, the computational efficiency of SDP based relaxations highly depends on the efficiency of the available SDP solvers.
While  recent progress \citep{zhao2010newton,yang2015sdpnal, sun2017sdpnal+} in solving large-scale SDPs
allows one to solve the SDP relaxation problem for clustering 2--3 thousand points, it is however prohibitively expensive
to solve the problem when $n$ goes beyond $3000$.

The work in \citep{Chi15} has implicitly demonstrated that it is generally much cheaper to solve the model  (\ref{Eq: Modified_Model}) instead of the SDP relaxation model. However, based on our numerical experiments, the algorithms ADMM and AMA proposed in \citep{Chi15} for solving (\ref{Eq: Modified_Model}) only work efficiently when the number of data points is not too large (several thousands depending on the feature dimension of the data). Also, it is not easy for the proposed algorithms in \citep{Chi15} to achieve relatively high accuracy. This also explains why we need to design a new algorithm in this paper to overcome the aforementioned
difficulties.

\section{Preliminaries and Notation}\label{sect:Preli}
In this section, we first introduce some preliminaries and notation which will be used later in this paper. For theoretical analysis, we adopt some definitions and notation from \citep{NIPS2014_5307, pmlr-v70-panahi17a}.
\begin{definition}
	For a given finite set $A = \{\mathbf{a}_1, \mathbf{a}_2, \dots, \mathbf{a}_n\} \subset \mathbb{R}^d$ and its partitioning $\mathcal{V} = \{V_1, V_2, \dots, V_{K}\}$, where each $V_i$ is a subset of $A$.
	\\[5pt]
(a) We say that a map $\psi$ on $A$ perfectly recovers $\mathcal{V}$ when $\psi(\mathbf{a}_i) = \psi(\mathbf{a}_j)$ is equivalent to $\mathbf{a}_i$ and $\mathbf{a}_j$ belonging to the same cluster. In other words, there exist distinct vectors $\mathbf{v}_1, \mathbf{v}_2, \dots, \mathbf{v}_K$ such that $\psi(\mathbf{a}_i) = \mathbf{v}_{\alpha}$ holds whenever $\mathbf{a}_i \in V_{\alpha}$.
\\[5pt]
(b) We call a partitioning $\mathcal{W} = \{W_1, W_2, \dots, W_L\}$ of $A$  a coarsening of $\mathcal{V}$ if each partition $W_l$ is obtained by taking the union of a number of partitions in ${\cal V}$. Furthermore, $\mathcal{W}$ is called the trivial coarsening of $\mathcal{V}$ if $\mathcal{W} = \{A\}$. Otherwise, it is called a non-trivial coarsening.
\end{definition}
\begin{definition}
	For any finite set $S \subset \mathbb{R}^d$, its diameter with respect to the
	$q$-norm for $q\geq 1$ is defined as
	$$D_q(S) := \max\{\|\mathbf{x} - \mathbf{y}\|_q \mid \mathbf{x}, \mathbf{y} \in S \}.$$
	Moreover, we define its separation and centroid, respectively, as
	$$
	d_q(S) := \min\{\|\mathbf{x} - \mathbf{y}\|_q \mid \mathbf{x}, \mathbf{y} \in S, \mathbf{x}
	\not = \mathbf{y}\},
	\qquad
	c(S) = \frac{\sum_{\mathbf{x}\in S}\mathbf{x}}{|S|}.
	$$
	For convenience, for any family of mutually disjoint finite sets $\mathcal{F} = \{F_i \subset \mathbb{R}^d\}$, we define $\mathcal{C}(\mathcal{F}) = \{c(F_i)\}$.
\end{definition}


Later in this paper, we will establish the theoretical recovery guarantee based on the above definitions. Next, we will introduce some preliminaries and notations for the design and analysis of the numerical optimization algorithms.

For a given simple undirected graph ${\cal G} = (\{1,\ldots,n\}, \mathcal{E})$ with $n$ vertices and edges defined in $\mathcal{E}$, we define the symmetric adjacency matrix $G \in \R^{n\times n}$ with entries
$$G_{ji} = G_{ij} = \left\{
\begin{array}{lc}
1 & if \; (i , j)\in\mathcal{E},\\
0 & \mbox{otherwise.}
\end{array}
\right.
$$
Based on an enumeration of the index pairs in $\mathcal{E}$ (say in the lexicographic order), which we
	denote by $l(i,j)$ for the pair $(i,j)$,
we define the node-arc incidence matrix $\mathcal{J} \in \R^{n\times|\mathcal{E}|}$ as
\begin{equation}\label{eq: construct_J}
\mathcal{J}_{k}^{l(i , j)}= \left\{
\begin{array}{rl}
1 & \mbox{if $k = i$},\\
-1 & \mbox{if $k = j$},\\
0 & \mbox{otherwise},
\end{array}
\right.
\end{equation}
where $\mathcal{J}^{l(i , j)}_k$ is the $k$-th entry of the $l(i , j)$-th column of $\mathcal{J}_k$.
	
\begin{proposition}\label{Prop: 1}
	With matrices $G$, $\mathcal{J}$ defined above, we have the following results
	\begin{equation}
	\mathcal{J}\mathcal{J}^T  = {\rm diag}(G\mathbf{e})- G =: L_G,
	\end{equation}
	where  $\mathbf{e} \in \R^n$ is the column vector of all ones,
	and
$L_G$ is the Laplacian matrix associated with the
adjacency matrix $G$.
\end{proposition}
Now, for given variables $X \in \R^{d \times n}$, $Z \in \R^{d \times |\mathcal{E}|}$ and the graph $G$, we define the linear map $\mathcal{B}: \R^{d \times n} \to \R^{d \times |\mathcal{E}|}$ and its adjoint  $\mathcal{B}^*: \R^{d \times |\mathcal{E}|} \to \R^{d \times n}$,
respectively, by
\begin{eqnarray}
\mathcal{B}(X) &=& [(\mathbf{x}_i - \mathbf{x}_j)]_{(i , j)\in\mathcal{E}} = X \mathcal{J},
\label{eq-B}
\\[5pt]
\mathcal{B}^*(Z) &=& Z \mathcal{J}^T.
\label{eq-Bt}
\end{eqnarray}
Thus, by Proposition \ref{Prop: 1}, we have
\begin{equation}
\mathcal{B}^*(\mathcal{B}(X)) = X \mathcal{J}\mathcal{J}^T = XL_{G}.
\label{eq-BTB}
\end{equation}
For a given proper and closed convex function $p : \mathcal{X} \to (-\infty, +\infty]$, its proximal mapping $\Prox_{tp}(x)$ for $p$ at any $x \in \mathcal{X}$ with $t > 0$ is defined by
\begin{equation}\label{Eq: prox_map}
\Prox_{tp}(x) = \arg\min_{u \in \mathcal{X}} \{tp(u) + \frac{1}{2}\|u - x\|^2\}.
\end{equation}
In this paper, we will often make use of the following Moreau identity (See \citet{bauschke2011convex}[Theorem 14.3(ii)])
$$
\Prox_{tp}(x) + t\Prox_{p^*/t}(x/t) = x,
$$
where $t > 0$  and $p^*$ is the conjugate function of $p$.
It is well known that
proximal mappings are important for designing optimization algorithms and they have been well studied.
The proximal mappings for many commonly used functions have closed form formulas. Here, we summarize
those that are related to this paper in Table \ref{tab: prox_map}. In the table,
$\Pi_{C}$ denotes the projection onto a given closed convex set $C$.
\begin{table}[!h]
	\caption{Proximal maps for selected functions}
	\label{tab: prox_map}
	\begin{center}
		\begin{small}			
			\begin{tabular}{lll}
				\toprule
				$p(\cdot)$ & $\Prox_{tp}(\mathbf{x})$ & Comment\\
				\midrule
				$\|\cdot\|_1$ & $\left[1 - \frac{t}{|\mathbf{x}_l|}\right]_+ \mathbf{x}_l$ & Elementwise soft-thresholding\\
				$\|\cdot\|_2$ & $\left[1 - \frac{t}{\|\mathbf{x}\|_2}\right]_+ \mathbf{x}$ & Blockwise soft-thresholding\\
				$\|\cdot\|_{\infty}$ & $\mathbf{x} - \Pi_{t\mathcal{S}}(\mathbf{x})$ & $\mathcal{S}$ is the unit $\ell_1$-ball\\
				\bottomrule
			\end{tabular}			
		\end{small}
	\end{center}
	\vskip -0.1in
\end{table}

\section{Theoretical Guarantee of Convex Clustering Models}

The empirical success of the convex clustering model (\ref{Eq: Origin_Model}) has
strongly motivated researchers to investigate its theoretical clustering recovery guarantee. The perfect recovery results for convex clustering model (\ref{Eq: Origin_Model}), where all pairwise differences are considered with equal weights, have been proved by \citet{NIPS2014_5307} for the 2-clusters case and later by
\citet{pmlr-v70-panahi17a} for the $k$-clusters case.
\citet{tan2015statistical} analyzed the statistical properties of model (\ref{Eq: Origin_Model}) and \citet{radchenko2017convex} analyzed the statistical properties of model (\ref{Eq: Origin_Model}) with the
$\ell_1$-regularization term. In practice, many researchers (e.g. \cite{tan2015statistical, Chi15}) have
suggested the use of  the model  (\ref{Eq: Modified_Model}), which is not only computationally more attractive
but also lead to more robust clustering results. However, so far no theoretical guarantee has been provided for the convex clustering model with general weights. In this section, we first review the nice theoretical results proved by \citet{NIPS2014_5307} and \cite{pmlr-v70-panahi17a} for (\ref{Eq: Origin_Model}), and then we will present our new theoretical guarantee for the more challenging case of the general weighted convex clustering model
 (\ref{Eq: Modified_Model}).

\subsection{Theoretical Recovery Guarantee of Convex Clustering Model (\ref{Eq: Origin_Model})}

The first theoretical result by \citet{NIPS2014_5307}
guarantees the perfect recovery
 of (\ref{Eq: Origin_Model}) for the two-clusters case when the data in each cluster are contained in a cube
 and the two cubes are sufficiently well separated.
%
More recently, much stronger theoretical results have been established by
\citet{pmlr-v70-panahi17a} wherein the authors
proved the theoretical recovery guarantee of the fully uniformly weighted
model (\ref{Eq: Origin_Model})  for the general $k$-clusters case.

\begin{theorem}[\citet{pmlr-v70-panahi17a}]\label{Theorem: Theory-17-ICML}
Consider a finite set $A = \{\mathbf{a}_i \in \mathbb{R}^d \mid i = 1, 2, \dots, n\}$ of vectors and its partitioning $\mathcal{V} = \{V_1, V_2, \dots, V_{K}\}$.
For the SON model in (\ref{Eq: Origin_Model}), denote its optimal solution by $\{\bar{\mathbf{x}}_i\}$ and define the map $\phi(\mathbf{a}_i) = \bar{\mathbf{x}}_i$, $i=1,\ldots,n$.
\begin{itemize}
	\item[(i)] If $\gamma$ is chosen such that
	$$\max_{V \in \mathcal{V}}\frac{D_2(V)}{|V|} \;\leq \;\gamma \leq \frac{d_2(\mathcal{C}(\mathcal{V}))}{2n\sqrt{K}},
	$$
	then the map $\phi$ perfectly recovers $\mathcal{V}$.
	\item[(ii)] If $\gamma$ satisfies the following inequalities,
	$$\max_{V \in \mathcal{V}}\frac{D_2(V)}{|V|} \leq \gamma \leq \max_{V \in \mathcal{V}}\frac{\|c(A) - c(V)\|_2}{|A| - |V|},$$
	then the map $\phi$ perfectly recovers a non-trivial coarsening of $\mathcal{V}$.
\end{itemize}
\end{theorem}
It was shown in \citep{pmlr-v70-panahi17a} that one can treat the theoretical results in \citep{NIPS2014_5307} as a special case of Theorem \ref{Theorem: Theory-17-ICML}.

We shall see in the next subsection that
we can improve the upper bound
in part (i) of Theorem \ref{Theorem: Theory-17-ICML}
to $\gamma \leq \frac{d_2(\mathcal{C}(\mathcal{V}))}{2n}$, as a special case of our new theoretical results.

\subsection{Theoretical Recovery Guarantee of the Weighted Convex Clustering Model (\ref{Eq: Modified_Model})}

Although the convex clustering model (\ref{Eq: Origin_Model}) with the fully uniformly weighted
regularization has the nice theoretical recovery guarantee, it is usually computationally too expensive
to solve since the number of terms in the regularization grows quadratically with the number of data points $n$.
In order to reduce the computational burden,
in practice many researchers have proposed to use
the partially weighted convex clustering model  (\ref{Eq: Modified_Model})
described in the Introduction.
Moreover, they have observed better empirical performance of (\ref{Eq: Modified_Model}) with well chosen weights, comparing to the original  model (\ref{Eq: Origin_Model}) \citep{Hocking11, Lindsten11, Chi15}. However, to the best of our knowledge, so far no theoretical recovery results have been established for the general weighted convex clustering model (\ref{Eq: Modified_Model}). Here we will prove that under rather mild conditions, perfect recovery can be
guarantee for  the weighted model (\ref{Eq: Modified_Model}).
In additional, our  theoretical results  subsume the known results for the
fully uniformly weighted model (\ref{Eq: Origin_Model}) as special cases.\\
Next, we will establish the main theoretical results for (\ref{Eq: Modified_Model}).
Our results and part of the proof have been inspired by the ideas used in \citep{pmlr-v70-panahi17a}.
For convenience, we define the index sets
$$I_{\alpha} := \{i \mid \mathbf{a}_i \in V_{\alpha}\}, \; \; {\rm for} \; \alpha = 1, 2, \dots, K.$$
Let $n_\alpha = |I_\alpha|$,
\begin{eqnarray*}
&&\mb{a}^{(\alp)} = \frac{1}{n_\alpha}\sum_{i\in I_\alp} \mb{a}_i,\quad
w^{(\alpha, \beta)} = \sum_{i \in I_{\alpha}}\sum_{j \in I_{\beta}} w_{ij}, \quad \forall\; \alpha,\beta=1,\ldots,K
\\[5pt]
&&w^{(\beta)}_i = \sum_{j\in I_\beta} w_{ij}, \quad \forall\; i=1,\ldots,n,\; \beta=1,\ldots, K.
\end{eqnarray*}
Here we will interpret $w^{(\beta)}_i$ as the coupling between point $\mb{a}_i$ and the $\beta$-th cluster,
and $w^{(\alp,\beta)}$ as the coupling between the $\alp$-th and $\beta$-th clusters.
We also define for $p\geq 1$,
		$$h(\mathbf{v}) := \|\mathbf{v}\|_p  = \Big(\sum_{i=1}^{d}|v_i|^p\Big)^{\frac{1}{p}},
		\quad \mathbf{v} = (v_1, v_2, \dots, v_d) \in \mathbb{R}^d,
		$$
		and note that the subdifferential of $h(\mathbf{v})$ is given by
		$$
		\partial h(\mathbf{v}) = \left\{
		\begin{array}{ll}
		\{\mathbf{y} \in \mathbb{R}^d \mid \|\mathbf{y}\|_q \leq 1, \langle \mathbf{y}, \mathbf{v} \rangle = \|\mathbf{v}\|_p\}& \mbox{if $\mathbf{v} \not = 0$},\\[5pt]
		\{\mathbf{y} \in \mathbb{R}^d \mid \|\mathbf{y}\|_q \leq 1 \} & \mbox{if $\mathbf{v} = 0,$}
		\end{array}
		\right.
		$$
		where $q \geq 1$ is the conjugate index of $p$ such that $\frac{1}{p} + \frac{1}{q} = 1$. Observe that for any $\by \in \partial h(\bv)$, we have $\|\by\|_q \leq 1$.\\[5pt]
\begin{theorem}
		\label{Theorem-weighted-SON}
		Consider an input data $A = [\mathbf{a}_1, \mathbf{a}_2, \dots, \mathbf{a}_n] \in \mathbb{R}^{d \times n}$ and its partitioning $\mathcal{V} = \{V_1, V_2, \dots, V_K \}$. Assume that all the centroids $\{\mathbf{a}^{(1)}, \mathbf{a}^{(2)}, \dots, \mathbf{a}^{(K)}\}$ are distinct. Let $q \geq 1$ be the conjugate index of $p$ such that $\frac{1}{p} + \frac{1}{q} = 1$. Denote the optimal solution of (\ref{Eq: Modified_Model}) by $\{\mathbf{x}_i^{*}\}$ and define the map $ \phi(\mathbf{a}_i) = \mathbf{x}_i^{*}$ for $i=1,\ldots,n.$
		\begin{itemize}
			\item[1.]
			Let
			$$
			\mu^{(\alp)}_{ij} : =
			\sum_{\beta=1,\beta\not=\alp}^K \Big| w^{(\beta)}_i - w^{(\beta)}_j
			\Big| , \quad i, j \in I_{\alp},\; \alp = 1, 2, \dots, K.
			$$
			Assume that $w_{ij} > 0$ and $n_{\alp}w_{ij} > \mu_{ij}^{(\alp)}$ for all $i, j \in I_\alp$, $\alp=1,\ldots,K$. Let
			\begin{equation}
			\label{Eq: Sufficent Condition}
			\begin{array}{l}
			\gamma_{\min} \;:=\;
			\max_{1\leq \alp\leq K} \max_{i,j\in I_\alp} \left\{
			\frac{\norm{\mathbf{a}_i - \mathbf{a}_j}_{q}}{n_{\alp}w_{ij} - \mu^{(\alp)}_{ij}}
			\right\},
			\\[10pt]
			 \gamma_{\max} :=
			\min_{1\leq \alp < \beta \leq K} \Big\{\frac{\norm{\mathbf{a}^{(\alp)} -\mathbf{a}^{(\beta)}}_{q}}{\frac{1}{n_{\alp}}\sum_{1 \leq l \leq K, l \not = \alp}w^{(\alp, l)} +  \frac{1}{n_{\beta}}\sum_{1 \leq l \leq K, l \not = \beta}w^{(\beta, l)}}\Big\}.
			\end{array}
			\end{equation}
			If $\gam_{\min} < \gam_{\max}$ and  $\gamma$ is chosen such that $\gamma \in [\gamma_{\min},\gamma_{\max})$,
			then the map $\phi$ perfectly recovers $\mathcal{V}$.
			\item[2.] If $\gamma$ is chosen such that
			$$
			\gam_{\min} \; \leq \; \gamma \;
			< \max_{1\leq \alpha\leq K}\frac{n_{\alp}\|\mathbf{c} - \mathbf{a}^{(\alp)}\|_q}{\sum_{ 1\leq\beta\leq K, \beta\not = \alp}w^{(\alp, \beta)}},$$
			where  $\mathbf{c} = \frac{1}{n}\sum_{i=1}^n \mathbf{a}_i$, then the map $\phi$ perfectly recovers a non-trivial coarsening of $\mathcal{V}$.
		\end{itemize}
	\end{theorem}
	
	\begin{proof} First
		we introduce the following centroid optimization problem corresponding to (\ref{Eq: Modified_Model}):
		\begin{eqnarray}
		\min \Big\{  \frac{1}{2} \sum_{\alp=1}^K  n_\alp \norm{\mathbf{x}^{(\alp)}-\mathbf{a}^{(\alp)} }^2
		+ \gamma \sum_{\alp=1}^K \sum_{\beta=\alp + 1}^K   w^{(\alp,\beta)}
		\norm{\mathbf{x}^{(\alp)} - \bx^{(\beta)}}_{p} \mid \mb{x}^{(1)},\ldots, \mb{x}^{(K)} \in \R^d
		\Big\}.
		\label{eq-centroid-lp}
		\end{eqnarray}
		Denote the optimal solution of (\ref{eq-centroid-lp}) by $\{\bar{\bx}^{(\alp)} \mid \alp= 1, 2, \dots, K\}$.
		The proof will rely on the relationships between (\ref{Eq: Modified_Model}) and (\ref{eq-centroid-lp}).
		\\[5pt]
		(1a) First we show that, if $\gamma < \gamma_{\max}$, then $\bar{\bx}^{(\alp)} \not = \bar{\bx}^{(\beta)}$ for all
			$\alp\not=\beta$. From the optimality condition of (\ref{eq-centroid-lp}), we have that
		\begin{eqnarray}
		n_\alp (\bar{\bx}^{(\alp)} - \ba^{(\alp)}) + \gamma \sum_{\beta=1,\beta\not=\alp}^K w^{(\alp,\beta)} \bar{\bz}^{(\alp,\beta)} \;=\; 0,
		\quad \forall\; \alp=1,\ldots,K,
		\label{eq-FOC-centroid-lp}
		\end{eqnarray}
		where
		$
		\bar{\bz}^{(\alp,\beta)} \in \partial h(\bar{\bx}^{(\alp)}-\bar{\bx}^{(\beta)}), \; \alp\not=\beta.
		$
		Now from (\ref{eq-FOC-centroid-lp}), we get for $\alp \not = \beta$,
		\begin{eqnarray*}
		\begin{array}{rll}
			& \bar{\bx}^{(\alp)} - \bar{\bx}^{(\beta)} & = \ba^{(\alp)} - \ba^{(\beta)} - \frac{\gamma}{n_{\alp}}\sum_{l=1, l\not = \alp}^{K}w^{(\alp, l)}\bar{\bz}^{(\alp, l)} + \frac{\gamma}{n_{\beta}}\sum_{l=1, l \not = \beta}^{K}w^{(\beta, l)}\bar{\bz}^{(\beta, l)}
			\\[8pt]
			\Rightarrow& \|\bar{\bx}^{(\alp)} - \bar{\bx}^{(\beta)}\|_q & \geq\;  \|\ba^{(\alp)} - \ba^{(\beta)}\|_q - \frac{\gamma}{n_{\alp}}\sum_{l=1, l\not = \alp}^{K}w^{(\alp, l)} \norm{\bar{\bz}^{(\alp, l)}}_q - \frac{\gamma}{n_{\beta}}\sum_{l=1, l \not = \beta}^{K}w^{(\beta, l)} \norm{\bar{\bz}^{(\beta, l)}}_q
			\\[8pt]
		& & \geq\; \|\ba^{(\alp)} - \ba^{(\beta)}\|_q - \gamma \Big(\frac{1}{n_{\alp}}\sum_{l=1, l\not = \alp}^{K}w^{(\alp, l)} +  \frac{1}{n_{\beta}}\sum_{l=1, l \not = \beta}^{K}w^{(\beta, l)}\Big)
		\\[8pt]
		&& \geq \; \|\ba^{(\alp)} - \ba^{(\beta)}\|_q \; \Big(1 - \frac{\gam}{\gam_{\max}} \Big) \; >\; 0.
		\end{array}
		\end{eqnarray*}
		Thus $\bar{\bx}^{(\alp)} \not = \bar{\bx}^{(\beta)}$ for all $\alp\not=\beta.$
\\[5pt]		
		(1b) Suppose  that $\gamma < \gamma_{\max}$. Then from (a),
		$\bar{\bx}^{(\alp)}\not=\bar{\bx}^{(\beta)}$ for all $\alp\not=\beta$. Next we prove that, if $\gam \geq \gamma_{\min}$, then
			\begin{eqnarray*}
				\bx^*_i = \bar{\bx}^{(\alp)}, \quad \forall\; i\in I_\alp, \quad \alp=1,\ldots,K
			\end{eqnarray*}
			is the unique optimal solution of (\ref{Eq: Modified_Model}). \\
			To do so, we start with the optimality condition for (\ref{Eq: Modified_Model}), which is given as follows:
		\begin{equation}
		\label{Eq: FOC-weighted-SON}
		\bx_i - \ba_i + \gamma \sum_{j=1, j\not = i}^n w_{ij}\bz_{ij} = 0, \; \; i = 1, 2, \dots, n,
		\end{equation}
		where $\bz_{ij} \in \partial h(\bx_{i} - \bx_{j})$. Consider
		\begin{eqnarray*}
			\bz_{ij}^* &=&
			 \left\{\begin{array}{ll}
				\bar{\bz}^{(\alp,\beta)} & \mbox{if $i\in I_\alp$, $j\in I_\beta$, $1\leq \alp,\beta\leq K$, $\alp\not=\beta$},
				\\[5pt]
				\frac{1}{n_{\alpha}w_{ij}}
				\Big[ \frac{1}{\gamma}(\ba_i-\ba_j) - (\bp_i^{(\alp)} - \bp_j^{(\alp)})\Big] &\mbox{if $i,j\in I_\alp$, $i\not=j$, $\alp=1,\ldots,K$},
			\end{array} \right.
		\end{eqnarray*}
		where
		$$
		\bp_i^{(\alp) }\;=\; \sum_{\beta=1,\beta\not=\alp}^K \left[ w^{(\beta)}_i		
		- \frac{1}{n_\alp} w^{(\alp,\beta)}
		\right] \bar{\bz}^{(\alp,\beta)}.
		$$
		We can readily prove that
		$$\|\bp^{(\alp)}_i - \bp_j^{(\alp)}\|_q \; \leq\;  \mu_{ij}^{(\alp)}$$
		and
		\begin{eqnarray*}
		\sum_{j \in I_{\alp}} \bp_j^{(\alp)} & = & \sum_{j \in I_{\alp}}\left(\sum_{\beta=1,\beta\not=\alp}^K \left[
		w^{(\beta)}_j
		- \frac{1}{n_\alp} w^{(\alp,\beta)}
		\right] \bar{\bz}^{(\alp,\beta)}\right)
		\\
		&=& \sum_{\beta=1,\beta\not=\alp}^K\left(\sum_{j \in I_{\alp}} \left[w^{(\beta)}_j
		- \frac{1}{n_\alp} w^{(\alp,\beta)}
		\right]\right) \bar{\bz}^{(\alp,\beta)}
		\;=\; \mathbf{0}.
		\end{eqnarray*}
		For convenience, we set $\bz_{ii} = 0$ for $i = 1, 2, \dots, n$. Now, we show that $\bz_{ij}^* \in \partial h(\bx_{i}^{*} - \bx_{j}^{*})$.\\
		If $i \in I_{\alp}$ and $j \in I_{\beta}$ for $\alp \not= \beta$, then we have that
		$$\bz_{ij}^* = \bar{\bz}^{(\alp, \beta)} \in \partial h(\bar{\bx}^{(\alp)} - \bar{\bx}^{(\beta)}) = \partial h(\bx_{i}^{*} - \bx_{j}^{*}).$$
		It remains to show that $\|\bz_{ij}^*\|_q \leq 1$ for all $i, j \in I_{\alp}, \alp = 1, 2, \dots, K$.  By direct calculations,  we have that for $\gam \geq \gam_{\min}$,
		\begin{eqnarray*}
			\|\bz_{ij}^*\|_q &=& \frac{1}{n_{\alp}w_{ij}}\Big\|\frac{1}{\gamma}(\ba_i - \ba_j) - (\bp_i^{(\alp)} - \bp_j^{(\alp)})\Big\|_q
			\;\leq\;
			 \frac{1}{\gamma n_{\alp}w_{ij}}\|\ba_i - \ba_j\|_q + \frac{1}{n_{\alp}w_{ij}}\mu_{ij}^{(\alp)}\\
			& \leq &  \frac{1}{n_{\alp}w_{ij}} (n_\alp w_{ij} - \mu^{(\alp)}_{ij})
			 + \frac{1}{n_{\alp}w_{ij}}\mu_{ij}^{(\alp)}
			\;=\; 1,
		\end{eqnarray*}
		which implies that $\bz_{ij}^* \in \partial h(\bx_{i}^{*} - \bx_{j}^{*}) = \partial h(\mb{0})$ for all $i, j \in I_{\alp}$.

		Finally, we show that the optimality condition (\ref{Eq: FOC-weighted-SON}) holds for $(\bx^*_1,\ldots,\bx_n^*)$.
		We have that for $i \in I_{\alp}$,
		\begin{eqnarray*}
			&& \hspace{-0.7cm}
			 \bx_i^{*} - \ba_i + \gamma\sum_{j=1, j\not = i}^n w_{ij}\bz_{ij}^*
			\;\;=\;\;  \bar{\bx}^{(\alp)} - \ba_i + \gamma\sum_{\beta=1}^K\sum_{j\in I_\beta}w_{ij}\bz_{ij}^*
			\\
			&=& \bar{\bx}^{(\alp)} - \ba^{(\alp)} + \gamma\sum_{\beta=1,\beta\not=\alp}^K \Big(\sum_{j\in I_\beta}w_{ij}\Big)\bar{\bz}^{(\alp, \beta)} + \ba^{(\alp)} - \ba_i + \gamma\sum_{j\in I_\alp}w_{ij}\bz_{ij}^*
			\\
			&=& \gamma\sum_{\beta=1,\beta\not=\alp}^K \Big[w^{(\beta)}_i - \frac{1}{n_{\alp}}w^{(\alp, \beta)}\Big]\bar{\bz}^{(\alp, \beta)} + \ba^{(\alp)} - \ba_i + \gamma\sum_{j \in I_{\alp}}w_{ij}\bz_{ij}^*
			\\
			&=& \gamma \bp_i^{(\alp)} + \ba^{(\alp)} - \ba_i 		
			+ \frac{\gam}{n_{\alp}}
			\sum_{j\in I_\alp}  \Big[\frac{1}{\gamma}(\ba_i-\ba_j) - (\bp_i^{(\alp)} - \bp_j^{(\alp)})\Big]
			\\
			&=& 0.
		\end{eqnarray*}
		Thus $(\bx^*_1,\ldots,\bx_n^*)$ is the optimal solution of (\ref{Eq: Modified_Model}).
Since $\phi(\ba_i) = \bx_i^* = \bar{\bx}^{(\alp)}$ for all $i\in I_\alp$, $\alp=1,\ldots, K$,
we see that the mapping $\phi$ perfectly recovers the clusters in ${\cal V}.$

		(2) Suppose on the contrary that $\bar{\bx}^{(1)} = \bar{\bx}^{(2)} = \cdots = \bar{\bx}^{(K)}$. Then, the optimal solution for (\ref{eq-centroid-lp}) degenerates to
		$$\bar{\bx} = \frac{1}{n}\sum_{i=1}^n \ba_i = \mathbf{c}.$$
		Thus, the optimality condition (\ref{eq-FOC-centroid-lp}) gives
		$$n_{\alp}\|\mathbf{c} - \ba^{\alp}\|_q \leq \gamma\sum_{\beta=1,\beta \not = \alp}^K
		w^{(\alp, \beta)}, \; \; \forall \;
		\alp \in \{1, 2, \dots, K\}.$$
		This implies that
		$$\gamma \;\geq\; \max_{1\leq\alp\leq K}\frac{n_{\alp}\|\mathbf{c} - \ba^{(\alp)}\|_q}{\sum_{\beta=1,\beta \not = \alp}^K w^{(\alp, \beta)}},$$
		which is a contradiction. Thus $\{ \bar{\bx}^{(1)},\ldots,\bar{\bx}^{(K)}\}$ must have a
		distinct pair.
\end{proof}

The above theorem has established the theoretical recovery guarantee for the general weighted convex clustering model
(\ref{Eq: Modified_Model}).
Later, we will demonstrate that the sufficient conditions
that $\gamma$ must satisfy is practically meaningful in the numerical experiments section. Now, we explain
the derived sufficient conditions intuitively.

For unsupervised learning, intuitively, we can get  meaningful clustering results when the given dataset has the properties that the elements within the same cluster are ``tight'' (in other words, the diameter should be small) and the centroids for different clusters are well separated.
Indeed, the conditions we have established are consistent with the intuition just discussed. First, the left-hand side in (\ref{Eq: Sufficent Condition}) characterizes the maximum weighted distance between the elements
in the same cluster. On the other hand, the right-hand side in (\ref{Eq: Sufficent Condition}) characterizes the minimum weighted distance between different centroids. Thus based on our discussion, we can expect perfect recovery to be practically possible for the weighted convex clustering model if the right-hand side is larger than the left-hand side in (\ref{Eq: Sufficent Condition}).
\begin{remark}
(a) Note that the assumption that $w_{ij} > 0$ is only  needed for all the pairs $(i,j)$ belonging to
the same cluster $I_\alp$ for all $1\leq \alp\leq K$. Thus the weights $w_{ij}$ can be chosen to be
zero if $i$ and $j$ belong to different clusters. As a result, the number of pairwise differences in the
regularization term can be much fewer than the total of $n(n-1)/2$ terms. This implies that
we can gain substantial computational efficiency when dealing with the sparse weighted regularization term.
\\[5pt]
(b) The quantity $\mu^{(\alp)}_{ij} = \sum_{\beta=1,\beta\not=\alp}^K | w^{(\beta)}_i - w^{(\beta)}_j|$,
for $i,j\in I_\alp$,
measures the total difference in the couplings between $\ba_i$ and $\ba_j$ with the $\beta$-th cluster for all
$\beta\not =\alp.$
\end{remark}

Next, we show that the results in Theorem \ref{Theorem: Theory-17-ICML} are special cases of our results.  Therefore, we also include the result in \citep{NIPS2014_5307} as a special case.
\begin{corollary} \label{Corollary:special}
	In (\ref{Eq: Modified_Model}), if we take $w_{ij}=1$ for all $1\leq i<j\leq n$, then the results in Theorem \ref{Theorem-weighted-SON} reduce to the following.
	\begin{itemize}
		\item[(i)] If
			\begin{eqnarray*}
			\max_{1\leq\alp\leq K}\;  \frac{D_q(V_\alp)}{|V_\alp|}
			 \; \leq \; \gamma \;<\; \min_{1\leq \alp,\beta \leq K, \alp\not=\beta} \Big\{\frac{\norm{\mathbf{a}^{(\alp)} -\mathbf{a}^{(\beta)}}_{q}}{2n - n_{\alp} - n_{\beta}}\Big\},
			\end{eqnarray*}
			then the map $\phi$ perfectly recovers $\mathcal{V}$.			
		\item[(ii)] If
			$$\max_{1\leq \alp\leq K}\frac{D_q(V_\alp)}{|V_\alp|} \; \leq \; \gamma \; \leq\;
			 \max_{V \in \mathcal{V}}\frac{\|c(A) - c(V)\|_q}{|A| - |V|},$$
			then the map $\phi$ perfectly recovers a non-trivial coarsening of $\mathcal{V}$.
	\end{itemize}
\end{corollary}
\begin{proof}
	The results for this corollary follow directly from Theorem \ref{Theorem-weighted-SON} by
	noting that
	$D_q(V_\alp) = \max_{i,j\in I_\alp} \norm{\ba_i-\ba_j}_q/{n_\alp}$, and
	using the following facts for the special case:
	\begin{itemize}
		\item[(1)] $\mu_{ij}^{(\alp)} = \sum_{\beta=1,\beta\not=\alp}^K |w^{(\beta)}_i - w^{(\beta)}_j| =
		 \sum_{\beta=1,\beta\not=\alp}^K |n_\beta - n_\beta|  = 0,$ for all
		$i,j\in I_\alp$, $1\leq\alp\leq K.$
		\item[(2)] $\frac{1}{n_\alp} \sum_{\beta=1,\beta \not = \alp}^K w^{(\alp, \beta)} =
		\frac{1}{n_\alp} \sum_{\beta=1,\beta \not = \alp}^K n_\alp n_\beta =
		 n - n_{\alp},$ for all $1\leq \alp \leq K.$
	\end{itemize}
	We omit the details here.
\end{proof}

If we compare the upper bound we obtained for $\gamma$ in part (i) of
Corollary \ref{Corollary:special} to that obtained in
Theorem \ref{Theorem: Theory-17-ICML} by  \citep{pmlr-v70-panahi17a}
for the case $p=2$ (and hence $q=2$), we can see that
our upper bound is more relax in the sense that
\begin{eqnarray*}
\min_{1\leq \alp,\beta \leq K, \alp\not=\beta} \Big\{\frac{\norm{\mathbf{a}^{(\alp)} -\mathbf{a}^{(\beta)}}_2}{2n - n_{\alp} - n_{\beta}} \Big\} \; >\;
\min_{1\leq \alp,\beta \leq K, \alp\not=\beta} \Big\{\frac{\norm{\mathbf{a}^{(\alp)} -\mathbf{a}^{(\beta)}}_2}{2n } \Big\} = \frac{d_2({\cal C}({\cal V}))}{2n} \;\geq \;  \frac{d_2({\cal C}({\cal V}))}{2n\sqrt{K}}.
\end{eqnarray*}

\section{A Semismooth Newton-CG Augmented Lagrangian Method for Solving (\ref{Eq: Modified_Model})  }

In this section, we introduce a fast convergent ALM  for solving the weighted convex clustering model (\ref{Eq: Modified_Model})\footnote{Part of the numerical algorithm described here has been published in the ICML 2018  paper \citep{yuan-icml-18}.}. For simplicity,  we will only focus on designing a highly efficient algorithm to solve (\ref{Eq: Modified_Model}) with $p = 2$. The other cases can be done in a similar way. In particular,  the same algorithmic design and implementation can be applied to the case $p=1$ or $p=\infty$ with no difficulty.

\subsection{Duality and Optimality Conditions}
From now on, we will focus on the following weighted convex clustering model with the $2$-norm:
\begin{eqnarray*}
\label{Eq: convex-clustering-l2}
\min_{X \in \mathbb{R}^{d \times n}} \frac{1}{2}\sum_{i=1}^n \|\bx_{i} - \ba_i\|^2 + \gamma\sum_{i < j}w_{ij}\|\bx_{i} - \bx_{j}\|_2.
\end{eqnarray*}
{By  ignoring the terms with $w_{ij} = 0$, we consider the following problem:}
\begin{equation}
\label{Eq: convex-clustering-l2-reduced}
\min_{X \in \mathbb{R}^{d \times n}} \frac{1}{2}\sum_{i=1}^n \|\bx_{i} - \ba_i\|^2 + \gamma\sum_{(i, j)\in \mathcal{E}}w_{ij}\|\bx_{i} - \bx_{j}\|_2,
\end{equation}
where $\mathcal{E}: = \{(i, j) \mid w_{ij} > 0\}$.\\
Now, we present the dual problem of (\ref{Eq: convex-clustering-l2-reduced}) and its Karush-Kuhn-Tucker (KKT) conditions. First, we write (\ref{Eq: convex-clustering-l2-reduced}) equivalently in the following compact form
$$
(P) \quad \min_{X, U}\Big\{\frac{1}{2}\|X-A\|^2 + p(U) \mid \mathcal{B}(X) - U = 0  \Big\},
$$
where $p(U) =  \gamma\sum_{(i, j)\in\mathcal{E}} w_{ij}\|U^{l(i, j)}\|$ and
${\cal B}$ is the linear map defined in (\ref{eq-B}). Here  $U^{l(i, j)}$ denotes the $l(i, j)$-th column of $U\in \R^{d\times |\mathcal{E}|}$.
The dual problem for ($P$) is given by
$$
(D) \quad \max_{V, Z}\Big\{\langle A, V \rangle - \frac{1}{2}\|V\|^2 \,\mid\,
\mathcal{B}^*(Z) - V = 0, Z \in \Omega\Big\},
$$
where $\Omega = \{Z \in \R^{d \times \,\mid\,\mathcal{E}|}\,\mid\,
\|Z^{l(i, j)}\| \leq \gamma w_{ij}, (i, j) \in \mathcal{E}\}$.
The KKT conditions for ($P$) and ($D$)  are given by
$$
(KKT) \quad \left\{
\begin{array}{ccc}
V + X - A &=& 0, \\[2pt]
U - \Prox_{p}(U + Z) &=& 0,\\[2pt]
\mathcal{B}(X) - U &=& 0,\\[2pt]
\mathcal{B}^*(Z) - V &=& 0.
\end{array}\right.
$$

\subsection{A Semismooth Newton-CG Augmented Lagrangian Method for Solving (P)}

In this section, we will design an inexact ALM for solving the primal problem $(P)$ but it will also solve $(D)$ as a byproduct.

We begin by defining the following Lagrangian function for $(P)$:
\begin{equation}
\label{eq: lagrangian}
l(X, U; Z) \;=\; \frac{1}{2}\|X - A\|^2 + p(U) + \langle Z, \mathcal{B}(X) - U\rangle.
\end{equation}
For a given parameter $\sigma > 0$, the augmented Lagrangian function associated with $(P)$ is given by
$$
\mathcal{L}_{\sigma}(X, U; Z) \;=\; l(X, U; Z) + \frac{\sigma}{2}\|\mathcal{B}(X) - U\|^2.
$$
The algorithm for solving $(P)$ is described in {\bfseries{Algorithm}} \ref{alg:ssnal}. To ensure the convergence of the inexact ALM in {\bfseries{Algorithm}} \ref{alg:ssnal}, we need the following stopping criterion for solving the subproblem (\ref{eq: primal_update}) in each iteration:
 \begin{equation}
(A) \quad {\rm dist}(0, \partial\Phi_{k}(X^{k+1}, U^{k+1})) \leq \epsilon_k/\max\{1, \sqrt{\sigma_k}\},
\end{equation}
where $\{\epsilon_k\}$ is a given summable sequence of nonnegative numbers.
\begin{algorithm}[!h]
	\caption{{\sc Ssnal} for $(P)$}
	\label{alg:ssnal}
	\begin{algorithmic}
		\STATE {\bfseries Initialization:} Choose $(X^0, U^0) \in \R^{d \times n} \times \R^{d \times |\mathcal{E}|}$,
		$ Z^0 \in \R^{d \times |\mathcal{E}|}$, $\sigma_0 > 0$ and a summable nonnegative sequence $\{\epsilon_k\}$.
		\REPEAT
		\STATE {\bfseries Step 1}. Compute
		\begin{equation}\label{eq: primal_update}
		(X^{k+1}, U^{k+1}) \approx \arg\min \{\Phi_{k}(X, U) = \mathcal{L}_{\sigma_k}(X, U; Z^k) \;\mid\;
		X\in \R^{d\times n}, \; U\in \R^{d\times |\cal E|} \}
		\end{equation}
		to satisfy the  condition (A) with the tolerance $\epsilon_k$.
		\STATE {\bfseries Step 2}. Compute
		$$Z^{k+1} = Z^k + \sigma_k(\mathcal{B}(X^{k+1}) - U^{k+1}).$$
		\STATE {\bfseries Step 3}. Update $\sigma_{k+1} \uparrow \sigma_{\infty} \leq \infty$.
		\UNTIL{Stopping criterion  is satisfied.}
	\end{algorithmic}
\end{algorithm}

Since a semismooth Newton-CG method will be used to solve the subproblems involved in the  above ALM
method, we call our algorithm a semismooth Newton-CG augmented Lagrangian method ({\sc Ssnal} in short).

\subsection{Solving the Subproblem (\ref{eq: primal_update})}
The inexact ALM is a well studied   algorithmic framework for solving convex  composite optimization problems. The key challenge in making the ALM efficient numerically is in solving the subproblem (\ref{eq: primal_update}) in each iteration efficiently to the required accuracy. Next, we will design a semismooth Newton-CG method to solve (\ref{eq: primal_update}). We will establish its quadratic convergence and develop sophisticated numerical techniques
to solve the associated semismooth Newton equations very efficiently by exploiting the underlying second-order structured sparsity in the subproblems.\\
For a given $\sigma$ and $\tilde{Z}$,  the subproblem (\ref{eq: primal_update}) in each iteration has the following
form:
\begin{equation}\label{eq: subproblem}
\min_{X \in \R^{d \times n}, U\in\R^{d\times |{\cal E}|}} \Phi(X, U) := \mathcal{L}_{\sigma}(X, U; \tilde{Z}).
\end{equation}
 Since $\Phi(\cdot, \cdot)$ is a strongly convex function,  the level set $\{(X, U) | \Phi(X, U) \leq \alpha\}$ is a closed and bounded convex set for any $\alpha \in \R$ and problem (\ref{eq: subproblem}) admits a unique optimal solution which we denote as $(\bar{X}, \bar{U})$. Now, for any $X$, denote
 $$
 \begin{array}{lcl}
 \phi(X) &:=& \inf_{U} \Phi(X, U) \;=\;  \frac{1}{2}\norm{X-A}^2  +\inf_U \Big\{ p(U) +\frac{\sigma}{2} \norm{U-\cB (X)-
\sigma^{-1}\tilde{Z}}^2  \Big\} - \frac{1}{2\sigma}\|\tilde{Z}\|^2
\\[5pt]
 &=& \frac{1}{2}\|X - A\|^2 + p(\Prox_{p/\sigma}(\mathcal{B}(X) + \sigma^{-1}\tilde{Z}))
 + \frac{1}{2\sigma}\| \Prox_{\sigma p^*}(\sigma\mathcal{B}(X) + \tilde{Z})\|^2 - \frac{1}{2\sigma}\|\tilde{Z}\|^2.
 \end{array}
 $$
 Therefore, we can compute  $(\bar{X}, \bar{U}) = \arg\min \Phi(X, U)$ by first computing
 $$
 \bar{X} = \arg\min_X \phi(X),
 $$
 and then compute $ \bar{U} = \Prox_{p/\sigma}(\mathcal{B}(\bar{X}) + \sigma^{-1}\tilde{Z}).$
 Since $\phi(\cdot)$ is strongly convex and continuously differentiable on $\R^{d \times n}$ with
 \begin{eqnarray}
 \nabla\phi(X) = X - A + \mathcal{B}^*(\Prox_{\sigma p^*}(\sigma\mathcal{B}(X) + \tilde{Z})),
 \label{eq-grad}
\end{eqnarray}
  we know that $\bar{X}$ can be obtained by solving the following nonsmooth equation
 \begin{equation}\label{eq: nonsmooth_eq}
 \nabla\phi(X) = 0.
 \end{equation}
It is well known that for solving smooth nonlinear equations, the quadratically convergent Newton's method is usually the first choice if it can be implemented
 efficiently. However, the usually required smoothness condition on $\nabla\phi(\cdot)$ is not satisfied  in our problem. This motivates us to develop a semismooth Newton method to solve the nonsmooth equation (\ref{eq: nonsmooth_eq}). Before we present our semismooth Newton method, we introduce the following definition of semismoothness, adopted from \citep{mifflin1977semismooth, kummer1988newton, qi1993nonsmooth}, which will be useful for analysis.
 \begin{definition}
 	(Semismoothness).  For a given open set $\mathcal{O} \subseteq \mathbb{R}^n$, let $F:\mathcal{O} \rightarrow \mathbb{R}^m$ be a locally Lipschitz continuous function and  $\mathcal{G}: \mathcal{O} \rightrightarrows \mathbb{R}^{m \times n}$ be a nonempty compact valued upper-semicontinuous multifunction. $F$ is said to be semismooth at $x \in \mathcal{O}$ with respect to the multifunction $\mathcal{G}$ if
 	$F$ is directionally differentiable at $x$ and for any $V \in \mathcal{G}(x + \Delta x)$ with $\Delta x\rightarrow 0$,
 	\[F(x+\Delta x) - F(x) - V\Delta x = o(\|\Delta x\|).\]
 	$F$ is said to be strongly semismooth at $x \in \mathcal{O}$ with respect to $\mathcal{G}$ if it is semismooth at $x$ with respect to $\mathcal{G}$ and
 	\[F(x+\Delta x) - F(x) - V\Delta x = O(\|\Delta x\|^2).\]
 	$F$ is said to be a semismooth (respectively, strongly semismooth) function on $\mathcal{O}$ with respect to $\mathcal{G}$ if it is semismooth (respectively, strongly semismooth) everywhere in $\mathcal{O}$ with respect to $\mathcal{G}$.
 \end{definition}
The following lemma shows that the proximal mapping of the $2$-norm is strongly semismooth with respect to its Clarke generalized Jacobian (See \citet{Clarke83}  [Definition 2.6.1]  for the definition of the Clarke generalized Jacobian).

\begin{lemma}[\cite{zhang2017efficient}, Lemma 2.1]
	\label{lemma: strongly-semismooth}
	For any $t > 0$, the proximal mapping ${\rm Prox}_{t\|\cdot\|_2}$ is strongly semismooth with respect to the Clarke generalized Jacobian $\partial {\rm Prox}_{t\|\cdot\|_2}(\cdot)$.
\end{lemma}
Next we derive the generalized Jacobian of the locally Lipschitz continuous function $\nabla \phi(\cdot)$. For any given $X \in \R^{d \times n}$, the following set-valued map is well defined:
\begin{eqnarray}
\hat{\partial}^2\phi(X) &:=& \{\mathcal{I} + \sigma\mathcal{B}^*\mathcal{V}\mathcal{B} \;|\; \mathcal{V} \in \partial \Prox_{\sigma p^*}(\tilde{Z} + \sigma\mathcal{B}X)\blue{\}}
\nonumber \\[3pt]
&=&\{\mathcal{I} + \sigma\mathcal{B}^*(\mathcal{I} - \mathcal{P})\mathcal{B}\; |\; \mathcal{P} \in \partial \Prox_{p/\sigma}(\mbox{$\frac{1}{\sigma}$}\tilde{Z} + \mathcal{B}X)\}, \qquad
\label{eq-hess}
\end{eqnarray}
where $\partial \Prox_{\sigma p^*}(\tilde{Z} + \sigma\mathcal{B}X)$ and $\partial \Prox_{p/\sigma}(\mbox{$\frac{1}{\sigma}$}\tilde{Z} + \mathcal{B}(X))$ are the Clarke generalized Jacobians of the Lipschitz continuous mappings $\Prox_{\sigma p^{*}}(\cdot)$ and $\Prox_{p/\sigma}(\cdot)$ at $\tilde{Z} + \sigma\mathcal{B}X$ and $\mbox{$\frac{1}{\sigma}$}\tilde{Z} + \mathcal{B}X$, respectively. Note that from \citep{Clarke83} [p.75] and \citep{JBHU} [Example 2.5], we have that
$$
\partial^2 \phi(X)(d) = \hat{\partial}^2\phi(X)(d), \ \ \forall d \in \R^{d \times n},
$$
where $\partial^2\phi(X)$ is the generalized Hessian of $\phi$ at $X$. Thus, we may use $\hat{\partial}^2 \phi(X)$ as the surrogate for $\partial^2\phi(X)$. Since $\mathcal{I} - \mathcal{P} = \mathcal{V} \in \partial \Prox_{\sigma p^*}(\cdot)$ is symmetric and positive semdefinite, the elements in $\hat{\partial}^2\phi(X)$ are positive definite, which guarantees that (\ref{eq: newton-linearsystem}) in {\bfseries{Algorithm}} \ref{alg:ssncg} is well defined.

Now, we can present our semismooth Newton-CG ({\sc Ssncg}) method for solving (\ref{eq: nonsmooth_eq}) and we could expect to get a fast superlinear or even quadratic convergence.
\begin{algorithm}[!h]
	\caption{{\sc Ssncg} for (\ref{eq: nonsmooth_eq})}
	\label{alg:ssncg}
	\begin{algorithmic}
		\STATE {\bfseries Initialization:} Given $X^{0} \in \R^{d \times n}$,
		$\mu\in (0, 1/2)$,  $\tau \in (0, 1]$, and $\bar{\eta}, \delta \in (0, 1)$.
		For $j = 0, 1, \dots$
		\REPEAT
		\STATE {\bfseries Step 1}.
		Pick an element $\mathcal{V}_j$ in $\hat{\partial}^2\phi(X^j)$ that is defined in (\ref{eq-hess}).
		Apply the conjugate gradient (CG) method to find an approximate solution $d^{j} \in \R^{d \times n}$ to
		\begin{equation}\label{eq: newton-linearsystem}
		\mathcal{V}_{j}(d) \approx -\nabla\phi(X^j)
		\end{equation}
		such that
		$\|\mathcal{V}_j(d^j) + \nabla\phi(X^j)\| \leq$ $\min(\bar{\eta}, \|\nabla\phi(X^j)\|^{1 + \tau})$.
		\STATE {\bfseries Step 2}. (Line Search) Set $\alpha_j = \delta^{m_j}$, where $m_j$ is the first nonnegative integer $m$ for which
		$$\phi(X^j + \delta^m d^j) \leq \phi(X^j) + \mu\delta^m\langle \nabla\phi(X^j), d^j\rangle.$$
		\STATE {\bfseries Step 3}. Set $X^{j+1} = X^j + \alpha_j d^{j}$.
		\UNTIL{Stopping criterion based on $\norm{\nabla\phi(X^{j+1})}$ is satisfied.}
	\end{algorithmic}
\end{algorithm}

\subsection{Using the Conjugate Gradient Method to Solve (\ref{eq: newton-linearsystem})}
\label{sec: CG}
In this section, we will discuss how to solve the very large (of dimension $dn\times dn$)
symmetric positive definite linear system (\ref{eq: newton-linearsystem}) to compute the Newton direction efficiently. As the matrix representation of the coefficient  linear operator $\mathcal{V}_j$ in (\ref{eq: newton-linearsystem}) is expensive to compute and factorize, we will adopt the conjugate gradient (CG) method to solve it.
It is well known that the convergence rate of the CG method depends critically on the condition
number of the coefficient matrix. Fortunately for our linear system (\ref{eq: newton-linearsystem}), the
coefficient linear operator typically has a moderate condition number since it satisfies
the following condition:
\begin{eqnarray*}
 I \;\preceq \; {\cal V}_j \;\preceq \;  I + \sigma \cB^*\cB \;\preceq \; (1+\sigma\lam_{\max}(L_G)) I,
\end{eqnarray*}
where $\lam_{\max}(L_G)$ denotes the maximum eigenvalue of the Laplacian matrix $L_G$ of the
graph ${\cal G}$, and the notation ``$A\preceq B$'' means that $B-A$ is symmetric positive semidefinite.
 It is known from \cite{AndersonMorley} that $\lam_{\max}(G) $ is at most 2 times
the maximum degree of the graph. In the numerical experiments, the maximum degree
of the graph is roughly equal to  the number of chosen $k$ nearest neighbors.
In those cases, the condition number of ${\cal V}_j$ is bounded independent of
$dn$, and 
provided that $\sigma$ is not too large, we can expect the CG method to converge rapidly
even when $n$ and/or $d$ are
large.

The computational cost for each CG step is highly dependent on the cost for computing the matrix-vector product $\mathcal{V}_j(\tilde{d})$ for any given $\tilde{d} \in \R^{d \times n}$. Thus we will need to analyze how this product can be computed efficiently.
Let $D : = \mathcal{B}X^j + \sigma^{-1}\tilde{Z}$.
For $(i , j) \in \mathcal{E}$, define
$$
\alpha_{ij} = \left\{
\begin{array}{ll}
\frac{\sigma^{-1} \gamma w_{ij}}{\|D^{l(i , j)}\|} & \mbox{if $\|D^{l(i , j)}\| > 0$},
\\[5pt]
\infty & \mbox{otherwise.}
\end{array}
\right.
$$
Note that for the given $D \in \R^{d \times |\mathcal{E}|}$, the cost for computing $\alpha$ is $O(d|{\mathcal{E}}|)$ arithmetic operations. For later convenience, denote
$$
\widehat{\mathcal{E}} = \{ (i,j)\in \mathcal{E}\mid \alpha_{ij} < 1\}.
$$
Now we choose $\mathcal{P} \in \partial \Prox_{p/\sigma}(D)$ explicitly.
We can take $\mathcal{P}: \R^{d\times |\mathcal{E}|} \to \R^{ d\times |\mathcal{E}|}$ that is defined by
$$
(\mathcal{P}({U}))^{l(i , j)} = \left\{
\begin{array}{ll}	
\alpha_{ij}\frac{\inprod{D^{l(i , j)}}{{U}^{l(i , j)}}}{\|D^{l(i , j)}\|^2}D^{l(i , j)} + (1 - \alpha_{ij}){U}^{l(i, j)} & \mbox{if $(i,j)\in \widehat{\mathcal{E}}$},
\\[5pt]
0 & \mbox{otherwise.}
\end{array}
\right.
$$
Thus to compute $\mathcal{V}_j(X) = X + \sigma\mathcal{B}^*\mathcal{B}(X) - \sigma \mathcal{B}^* \mathcal{P} \mathcal{B}(X) = X(I_n+\sigma L_G) -\sigma \mathcal{B}^* \mathcal{P} \mathcal{B}(X)$ efficiently for a given $X\in \R^{d\times n}$, we need  the efficient computation of $\mathcal{B}^* \mathcal{P} \mathcal{B}(X)$ by using the following proposition.

\begin{proposition} \label{prop-2}
Let $X\in \R^{d\times n}$ be given.
\\
(a)	Consider the symmetric matrix $M \in \R^{n \times n}$ defined by $M_{ij} = 1-\alpha_{ij}$ if $(i, j)\in \widehat{\mathcal{E}}$ and $M_{ij} = 0$ otherwise. Let $Y =  [M_{ij} (\mathbf{x}_i- \mathbf{x}_j)]_{(i , j)\in \mathcal{E}} = X\mathcal{M}$,
	where $\mathcal{M}$ is defined similarly as in
	(\ref{eq: construct_J})
	for the matrix $M$. Then we have
	\begin{eqnarray*}
		\mathcal{B}^* (Y) =  X L_M,
	\end{eqnarray*}
	where $L_M$ is the Laplacian matrix associated with $M$. The cost of computing the result $\mathcal{B}^*(Y)$ is $O(d|\widehat{\mathcal{E}}|)$ arithmetic operations.
\\[5pt]
(b)
Define $\rho \in \R^{|\mathcal{E}|}$ by
$$
\rho_{l(i , j)} := \left\{ \begin{array}{l}
\frac{\alpha_{ij} }{\|D^{l(i,j)}\|^2} \langle D^{l(i,j)}, \bx_i-\bx_j \rangle, \;\mbox{ if $(i , j) \in \widehat{\mathcal{E}}$},
\\[8pt]
0,\quad\mbox{otherwise}.
\end{array}
\right.
$$
For the given $D \in \R^{d \times |\mathcal{E}|}$, the cost for computing $\rho$ is $O(d|\widehat{\mathcal{E}}|)$ arithmetic operations. Let $W^{l(i , j)}  = \rho_{l(i , j)}D^{l(i , j)}$. Then,
$$\mathcal{B}^*(W) = W\mathcal{J}^T = D {\rm diag}(\rho)
\mathcal{J}^T.
$$
(c) The computing cost for $\mathcal{B}^*\mathcal{P}\mathcal{B}(X) = \mathcal{B}^*(Y) + \mathcal{B}^*(W)$ in total is $O(d|\widehat{\mathcal{E}}|)$.
\end{proposition}

\bigskip
With the above proposition, we can readily see that ${\cal V}_j(X)$ can be
computed in $O(d|\cE|) + O(d|\widehat{\cE}|)$ operations, where the
first term comes from computing $ X (I + \sigma L_G)$ and
the second term comes from computing $\sigma\cB \cP \cB^*(X)$
based on Proposition \ref{prop-2}.

Besides the algorithmic aspect, the next remark shows that the second-order information
gathered in the semismooth Newton method
can capture data points which are near to the boundary of a cluster
if we wisely choose the weights $w_{ij}$. We believe this is a very useful result since boundary points detection is a challenging problem in data science, especially in the high dimensional setting where locating boundary points is
challenging even if we know the labels of all the data points.

\begin{remark}
	If we choose the weights based on the k-nearest neighbors, for example, set
	$$
	w_{ij} = \left\{
	\begin{array}{ll}
	\exp(-\phi\|\mathbf{a}_i - \mathbf{a}_j\|^2)  & if \; (i, j)\in \mathcal{E},\\[5pt]
	0 & otherwise,
	\end{array}
	\right.
	$$
	where $\mathcal{E} = \cup_{i=1}^n \{(i, j) \mid  \mbox{$j$ is among $i$'s $k$-nearest}$ $\mbox{neighbors}, i< j\leq n\}$. Then, $\alpha_{ij} < 1$ means that $j$ is among $i$'s $k$-nearest  neighbors but do not belong to the same cluster as $i$.
	Naturally we expect there will only be a small number of such occurrences if $\gamma$ is properly chosen.
	Hence, $|\widehat{\mathcal{E}}|$ is expected to be much smaller than $|\mathcal{E}|$.
	 On the other hand, for $\alpha_{ij} \geq 1$, it means that points $i$ and $j$ are in the same cluster. This result implies that after we have solved the optimization problem (\ref{Eq: Modified_Model}) with a properly selected $\gamma$,
	$\alp_{ij} < 1$ indicates that point $i$ is near to the boundary of its cluster. Also, we can expect most of the columns of the matrix $\mathcal{P}(\mathcal{B}(X))$ to be zero since its number of non-zero
	columns is at most $|\widehat{\cE}|$. We call such a property inherited from the
	generalized Hessian of $\phi(\cdot)$ at $X$ as the {\bf second-order sparsity}.
	This  also explains why we are able to compute $\mathcal{B}^*\mathcal{P}\mathcal{B}(X)$ at a very low cost.
\end{remark}

\subsection{Convergence Results}
In this section, we will establish the convergence results for both {\sc Ssnal} and
{\sc Ssncg} under mild assumptions.
First, we present the following global convergence result of our proposed Algorithm {\sc Ssnal}.
\begin{theorem}\label{thm: convergence_alm}
	Let $\{(X^k, U^k, Z^k)\}$ be the sequence generated by Algorithm \ref{alg:ssnal} with stopping criterion $(A)$. Then the sequence $\{X^k\}$ converges to the unique optimal solution of $(P)$, and
	$\norm{\mathcal{B}(X^k)-U^k}$ converges to $0$.
	In addition, $\{Z^k\}$ is converges to an optimal solution $Z^{*} \in \Omega$ of $(D)$.
\end{theorem}
The above convergence theorem can be obtained from
\citep{rockafellar1976augmented,
	rockafellar1976monotone} without much difficulties.
Next, we state the convergence property for the semismooth Newton algorithm {\sc Ssncg} used to
solve the subproblems in Algorithm \ref{alg:ssnal}.

\begin{theorem}
	Let the sequence $\{X^j\}$ be generated by Algorithm {\sc Ssncg}. Then $\{X^j\}$ converges to the unique solution $\bar{X}$ of the problem in (\ref{eq: nonsmooth_eq}), and
    for $j$ sufficiently large,
	$$\|X^{j+1} - \bar{X}\| = O(\|X^j - \bar{X}\|^{1+\tau}),$$
	where $\tau \in (0,1]$ is a given constant in the algorithm, which is typically chosen to be $0.5$.
\end{theorem}
\begin{proof}
	From Lemma \ref{lemma: strongly-semismooth}, we know that $\Prox_{t\|\cdot\|_2}$ is strongly semismooth for any $t > 0$, together with the Moreau identity $\Prox_{tp}(x) + t\Prox_{p^*/t}(x/t) = x$, we know that
	$$\nabla\phi(X) = X - A + \mathcal{B}^*(\Prox_{\sigma p^*}(\sigma\mathcal{B}(X) + \tilde{Z})),$$
	is strongly semismooth. By \citep{zhao2010newton} [Proposition 3.3], we know that $d^j$ obtained in {\sc Ssncg} is a descent direction, which guarantees that
	the Algorithm {\sc Ssncg} is well defined.
	From \citep{zhao2010newton} [Theorem 3.4, 3.5], we can get the desired convergence results.
\end{proof}

\subsection{ {Generating   an initial  point}}
In our implementation, we use the following inexact alternating direction method of multipliers ({\sc iadmm})
developed in \cite{chen2017efficient} to generate an initial point
to warm-start {\sc Ssnal}. Note that with the global convergence result stated in Theorem \ref{thm: convergence_alm}, the performance of {\sc Ssnal} does not sensitively depend on the initial points, but it is still helpful if we can choose a good one.
\begin{algorithm}[!h]
	\caption{{\sc iadmm} for $(P)$}
	\label{alg:iadmm}
	\begin{algorithmic}
		\STATE {\bfseries Initialization:} Choose $\sigma > 0$, $(X^0, U^0, Z^0) \in \R^{d \times n} \times \R^{d \times |\mathcal{E}|} \times \R^{d \times |\mathcal{E}|}$, and a summable nonnegative  sequence $\{\epsilon_k\}$. For $k = 0, 1, \dots,$
		\REPEAT
		\STATE {\bfseries Step 1}. Let $R^k = A+\sigma\mathcal{B}^*(U^k-\sigma^{-1}Z^k)$.
		Compute
		\begin{eqnarray*}
			X^{k+1} &\approx& \arg\min_{X} \{\mathcal{L}_{\sigma}(X, U^k; Z^k)\},
			\\[3pt]
			U^{k+1} &=& \arg\min_{U}\{\mathcal{L}_{\sigma}(X^{k+1}, U; Z^k)\},
		\end{eqnarray*}
		where $X^{k+1}$ is an inexact solution satisfying the accuracy requirement that
		$\norm{(I_n+\sigma\mathcal{B}^*\mathcal{B}) X^{k+1} -R^k} \leq \epsilon_k.$
		\STATE {\bfseries Step 2}. Compute
		$$Z^{k+1} = Z^k + \tau\sigma_k(\mathcal{B}(X^{k+1}) - U^{k+1}),$$
		where  $\tau \in (0, \frac{1 + \sqrt{5}}{2})$ is typically chosen to be 1.618.
		\UNTIL{the stopping criterion is satisfied.}
	\end{algorithmic}
\end{algorithm}

Observe that in Step 1, $X^{k+1}$ is a computed solution for the following large linear system
of equations:
$$
	(I_n+\sigma\mathcal{B}^*\mathcal{B})X = R^k \quad \Longleftrightarrow\quad (I_n+\sigma L_G) X^T = (R^k)^T.
$$
To compute $X^{k+1}$, we can adopt a direct approach if the sparse Cholesky factorization of
$I_n+\sigma L_G$ (which only needs to be done once) can be computed at a moderate cost; otherwise
we can adopt an iterative approach by applying the conjugate gradient method to solve the above
fairly well-conditioned linear system.

\section{Numerical Experiments}

In this section, we will first demonstrate that the sufficient conditions we derived for perfect recovery
in Theorem \ref{Theorem-weighted-SON} is practical via a simulated example. Then, we will show the superior performance of our proposed algorithm {\sc Ssnal} on both simulated and real datasets, comparing to the popular algorithms such as  ADMM and AMA which are proposed in \citep{Chi15}. In particular, we will focus on the efficiency, scalability, and robustness of our algorithm for different values of $\gamma$. Also, we will show the performance of our algorithm on large datasets and unbalanced data. Previous numerical demonstration on the scalability and performance of (\ref{Eq: Modified_Model}) on large datasets is limited.
 The problem sizes of the instances tested in  \citep{Chi15} and other related papers are at most
several hundreds ($n \leq 500$ in \citep{Chi15}, $n \leq 600$ in \citep{pmlr-v70-panahi17a}), which are not
large enough to conclusively demonstrate the
scalability of the algorithms.
In this paper, we will present numerical  results for $n$ up to $\mathbf{200,000}$. We will also analyze the sensitivity of the computational efficiency of {\sc Ssnal} and {\sc AMA}, with respect to different choices of the parameters in (\ref{Eq: Modified_Model}), such as $k$ (the number of nearest neighbors) and $\gamma$.

We focus on solving (\ref{Eq: Modified_Model}) with $p = 2$ since the rotational invariance
of the $2$-norm makes it a robust choice in practice. Also, this case is more challenging than $p = 1$ or $p = \infty$.\footnote{Our algorithm can be generalized to solve (\ref{Eq: Modified_Model}) with $p=1$ and $p = \infty$ without much
difficulty.}
As the results reported in \citep{Chi15} have been regarded as the benchmark for the convex clustering model (\ref{Eq: Modified_Model}), we will compare our algorithm with the open source software {\sc cvxclustr}\footnote{https://cran.r-project.org/web/packages/cvxclustr/index.html} in \citep{Chi15}, which is an R package with key functions written in C. We write our code in {\sc Matlab} without any dedicated C functions. All our computational results are obtained from a desktop having 16 cores with 32 Intel Xeon E5-2650 processors at 2.6 GHz and 64 GB memory.

In our implementation, we stop our algorithm based on the following relative KKT residual:
$$\max\{\eta_P, \eta_D, \eta\} \leq \epsilon,$$
where\\
$$\eta_P = \frac{\|\mathcal{B}X - U\|}{1 + \|U\|}, \; \eta_D = \frac{\sum_{(i , j) \in \mathcal{E}}
	\max\{0, \|Z^{l(i, j)}\|_2 - \gamma w_{ij}\}}{1 + \|A\|},$$
$$\eta = \frac{\|\mathcal{B}^*(Z) + X - A\| + \|U - {\rm Prox}_{p}(U+Z)\|}{1 + \|A\| + \|U\|},$$
and $\epsilon > 0$ is a given tolerance. In our experiments, we set $\epsilon = 10^{-6}$ unless specified otherwise.
Since the numerical results reported in \citep{Chi15} have demonstrated the superior performance of AMA over ADMM, we will mainly compare our proposed algorithm with AMA. We note that {\sc cvxclustr} does not use the relative KKT residual as
its stopping criterion but used the duality gap in AMA and $\max\{\eta_P, \eta_D\} \leq \epsilon$ in ADMM. To make a fair comparison, we first solve (\ref{Eq: Modified_Model}) using {\sc Ssnal} with a given tolerance $\epsilon$,
and denote the primal objective value obtained as $P_{\rm Ssnal}$. Then, we run AMA in {\sc cvxclustr} and stop it as
soon as the computed  primal objective function value ($P_{\rm AMA}$) is close enough to $P_{\rm Ssnal}$, i.e.,
\begin{equation}\label{eq:stop_ama}
P_{\rm AMA} - P_{\rm Ssnal} \leq 10^{-6} P_{\rm Ssnal}.
\end{equation}
We note that since (\ref{Eq: Modified_Model}) is an unconstrained problem, the quality of the computed solutions can directly be compared based on the objective function values. We also stop  AMA if the maximum of $10^5$ iterations is reached.

When we generate the clustering path for the first parameter value of $\gamma$, we first run the {\sc Iadmm} introduced in Algorithm 3 for $100$ iterations to generate an initial point, then we use {\sc Ssnal} to solve (\ref{Eq: Modified_Model}). After that, we use the previously computed optimal solution for the lastest $\gamma$ as the initial point to warm-start
{\sc Ssnal} for solving the problem corresponding to the next $\gamma$. The same strategy is used in {\sc cvxclustr}.

\subsection{Numerical Verification of Theorem \ref{Theorem-weighted-SON}}

In this section, we demonstrate that the theoretical results we obtained in Theorem \ref{Theorem-weighted-SON} are practically meaningful by conducting numerical experiments on a simulated dataset with five clusters. We generate the five clusters randomly via a 2D Gaussian kernel. Each of the cluster has 100 data points, as shown in Figure \ref{fig-1}.

\begin{figure}[!h]
	\begin{center}
		\includegraphics[width=0.45\textwidth]{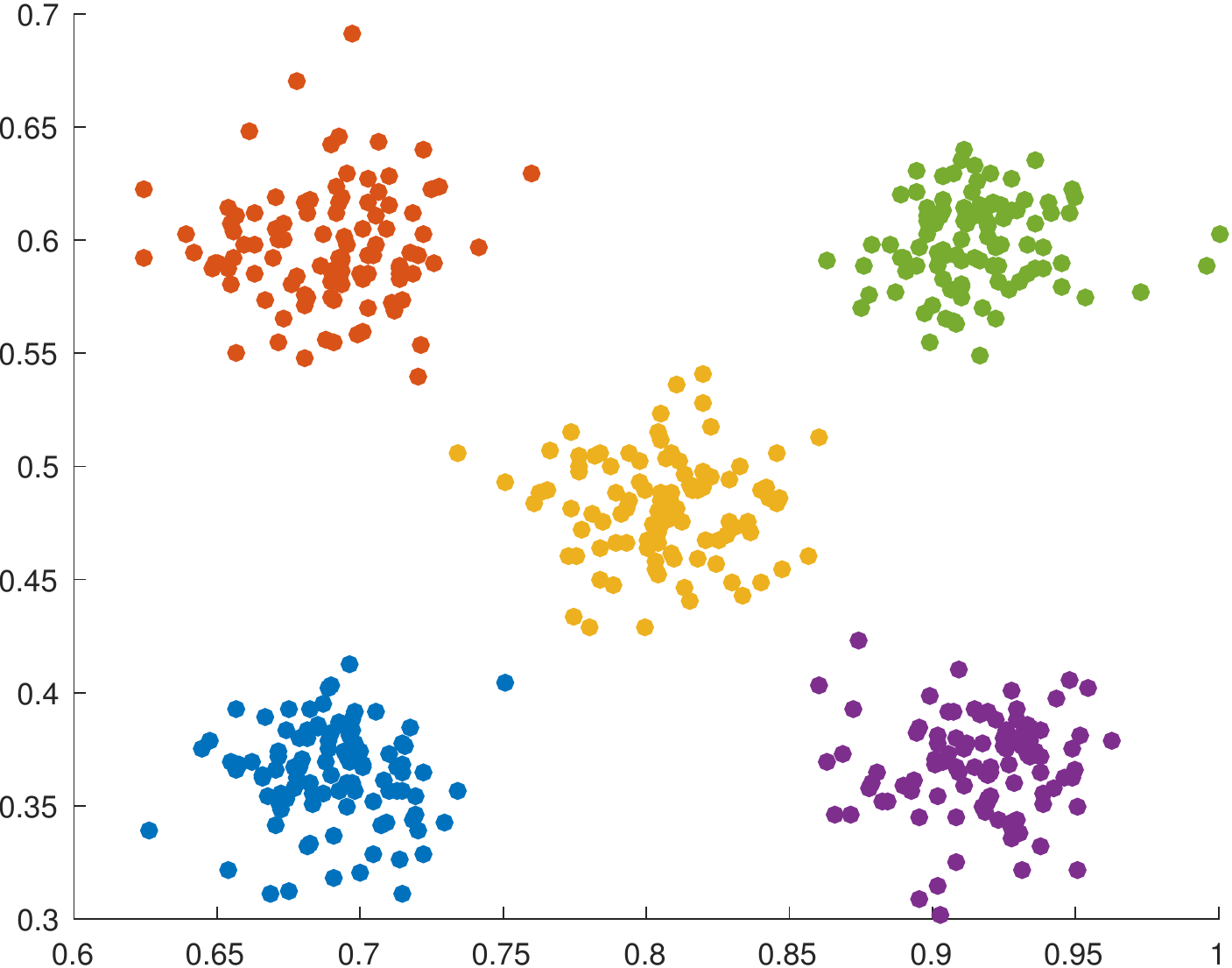}
	\end{center}
	\vskip -0.2in
		\caption{Visualization of the generated data.}
	\label{fig-1}
\end{figure}

Since we know the cluster assignment for each data point, we can construct the corresponding centroid problem
given in (\ref{eq-centroid-lp}). Then, we can solve the weighted convex clustering model (\ref{Eq: Modified_Model}) and the corresponding centroid problem (\ref{eq-centroid-lp}) separately to compare the results.
In our experiments, we choose the weight $w_{ij}$ as follows
$$
w_{ij} = \left\{
\begin{array}{ll}
\exp(-0.5\|\mathbf{a}_i - \mathbf{a}_j\|^2)
& \mbox{if  $(i, j)\in \mathcal{E}$},\\[5pt]
0 & \mbox{otherwise.}
\end{array}
\right.
$$
where $\mathcal{E} = \cup_{i=1}^n \{(i, j) \mid  \mbox{$j$ is among $i$'s 30-nearest}$ $\mbox{neighbors}, i< j\leq n\} \cup_{\alpha = 1}^5 \{(i, j) \mid i, j \in I_{\alpha}, i < j\}$.

First, we solve (\ref{Eq: Modified_Model}) and (\ref{eq-centroid-lp}) separately to find their optimal solutions, denoted as $X^* = [\bx_1^*, \bx_2^*, \dots, \bx_n^*]$ and $\bar{X} = [\bar{\bx}^{(1)}, \bar{\bx}^{(2)}, \dots, \bar{\bx}^{(K)}]$, respectively. Then, we can construct the new solution $\hat{X}$ for (\ref{Eq: Modified_Model}) based on $\bar{X}$ as
$$
\hat{\bx}_i = \bar{\bx}^{(\alp)} \; \; \forall\; i \in I_{\alp}, \quad \alp = 1,\ldots,5.
$$
We also compute the theoretical lower bound $\gamma_{\min}$ and upper bound $\gamma_{\max}$ based on the formula given in
Theorem \ref{Theorem-weighted-SON}, and they are given by
$$\gamma_{\min} = 1.56 \times 10^{-3}, \quad \gamma_{\max} = 0.485.$$

Based on the computed results shown  in the left panel of Figure \ref{fig-gamma}, we can observe the phenomenon that for very small $\gamma $, $X^*$ and $\hat{X}$ are different. However, when $\gamma$ becomes larger, $X^*$ and $\hat{X}$ coincide with each other in that $\norm{X^*-\hat{X}}$ is almost 0 (up to the accuracy level we solve the problems
(\ref{Eq: Modified_Model}) and (\ref{eq-centroid-lp})).
In fact, we see that for $\gamma$ larger than the theoretical lower bound $\gamma_{\min}$ but less than $\gam_{\max}$, we
have perfect recovery of the clusters by solving (\ref{Eq: Modified_Model}), and when $\gam$ is slightly
smaller than $\gamma_{\min}$, we lose the perfect recovery property.

Furthermore, from our results in Theorem \ref{Theorem-weighted-SON}, we know that when $\gamma$ is
smaller than $\gam_{\max}$ but larger than $\gam_{\min}$, we should recover the correct number of clusters. This is indeed observed
in the result shown in the right panel of Figure \ref{fig-gamma}
where we track the number of clusters for different values of $\gamma$.
Moreover, when $\gamma$ is about two times
larger than $\gamma_{\max}$, we get a coarsening of the clusters.
The results shown above demonstrate that the theoretical results we have established
in Theorem \ref{Theorem-weighted-SON} are meaningful in practice.

\begin{figure}[!h]
	\begin{center}
		\includegraphics[width=0.45\textwidth]{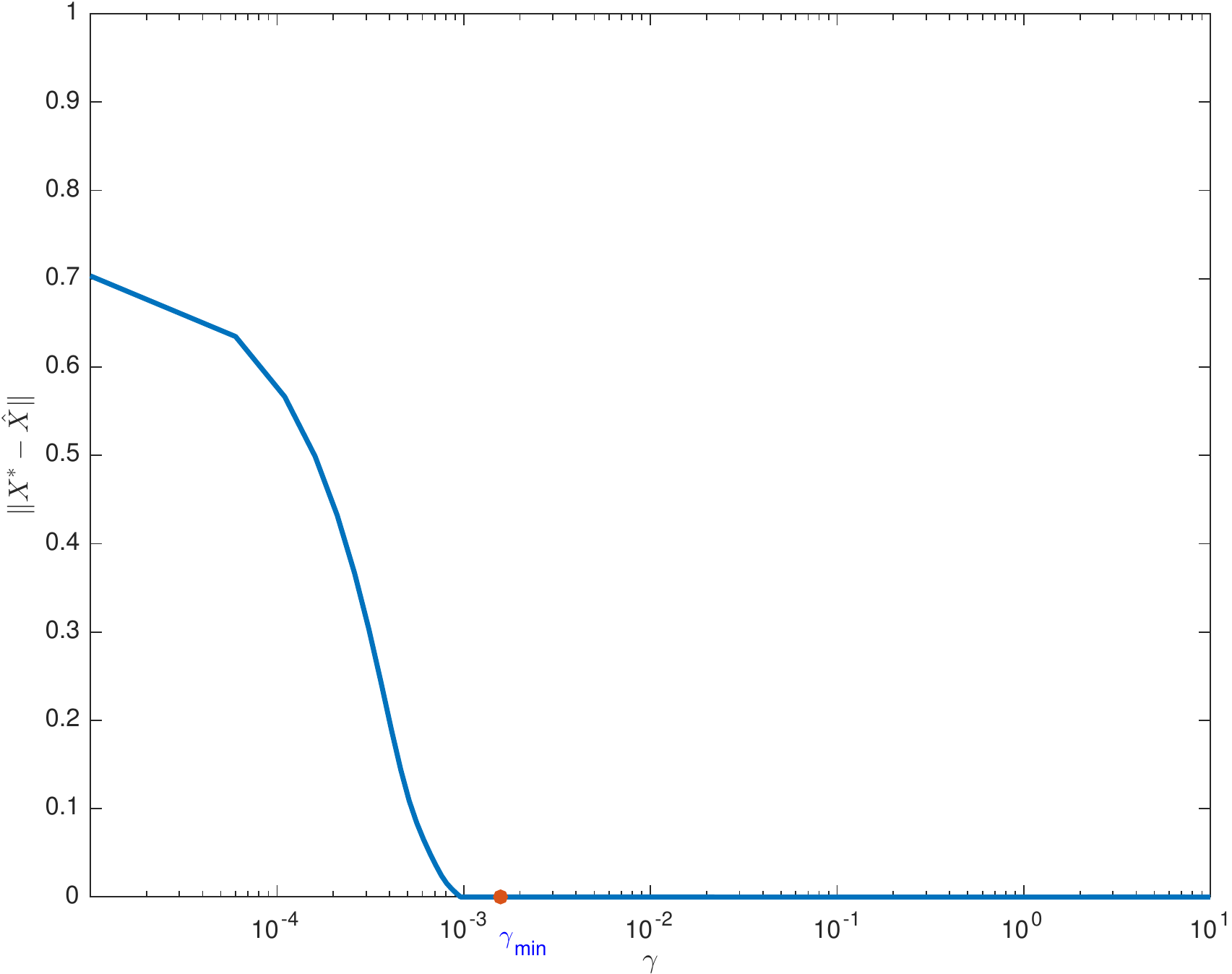}
		\qquad
		\includegraphics[width=0.45\textwidth]{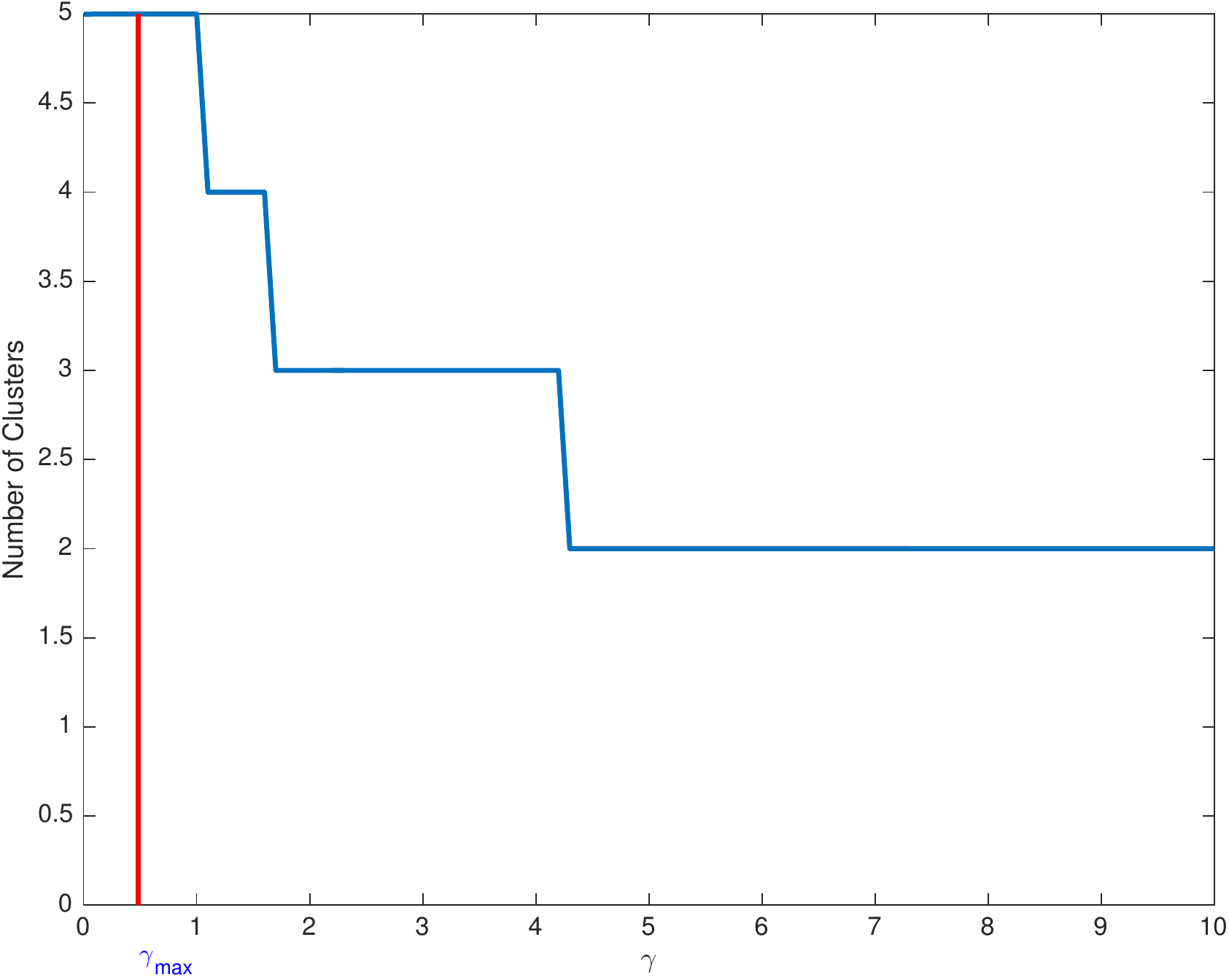}
	\end{center}
	\vskip -0.2in
	\caption{Left panel: $\|X^{*} - \hat{X}\|$ vs $\gamma$; Right panel: number of clusters vs $\gamma$.}
\label{fig-gamma}
\vskip-0.2in
\end{figure}

Next, we show the numerical performance of our proposed optimization algorithm for solving
(\ref{Eq: Modified_Model})  via
 (\ref{Eq: convex-clustering-l2-reduced}).

\subsection{Simulated data}

In this section, we show the performance of our algorithm {\sc Ssnal} on three simulated datasets: Two Half-Moon, Unbalanced Gaussian \citep{UnbalanceSet} and semi-spherical shells data. We compare our {\sc Ssnal} with the {\sc AMA} in \citep{Chi15} on different problem scales. The numerical results in Table \ref{tab: cmp_time_halfmoon}  show the superior performance of {\sc Ssnal}.
We also visualize some selected recovery results for Two Half-moon and Unbalanced Gaussian in Figure
\ref{fig: visualization_sim}\,.

\begin{figure*}[!h]
	\vskip 0.2in
	\begin{center}
			\includegraphics[width=0.3\textwidth]{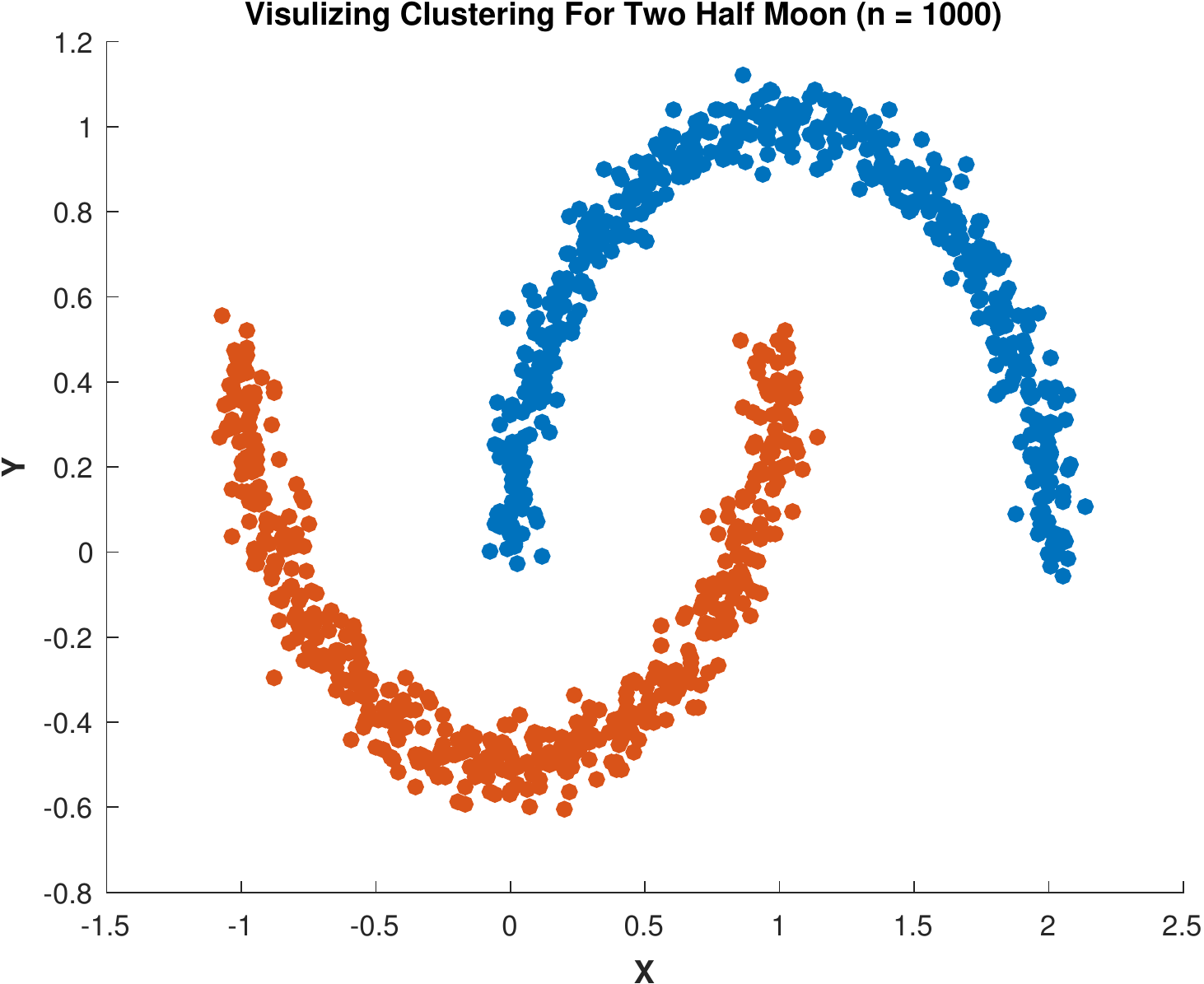}
			\includegraphics[width=0.3\textwidth]{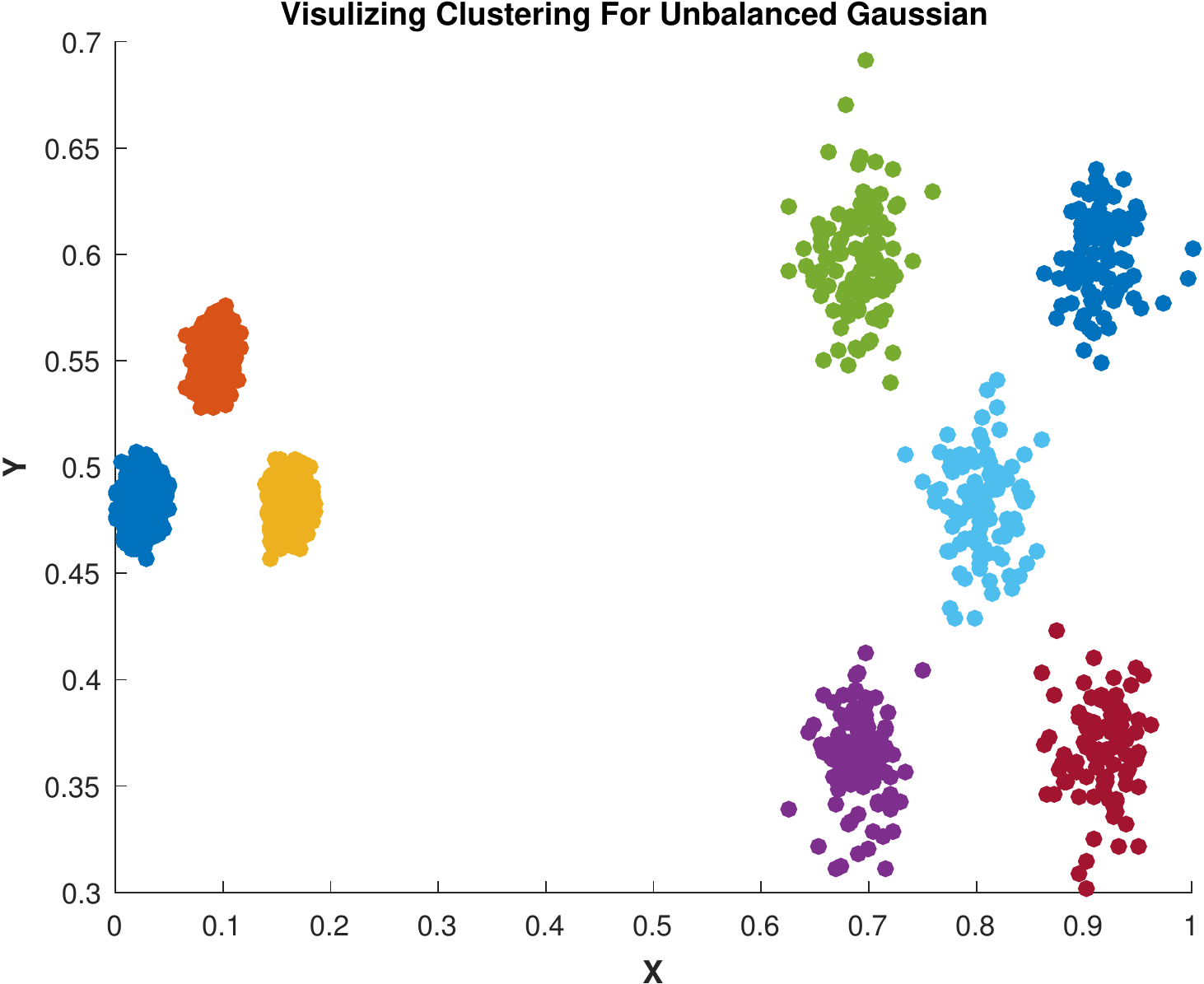}
			\includegraphics[width=0.3\textwidth]{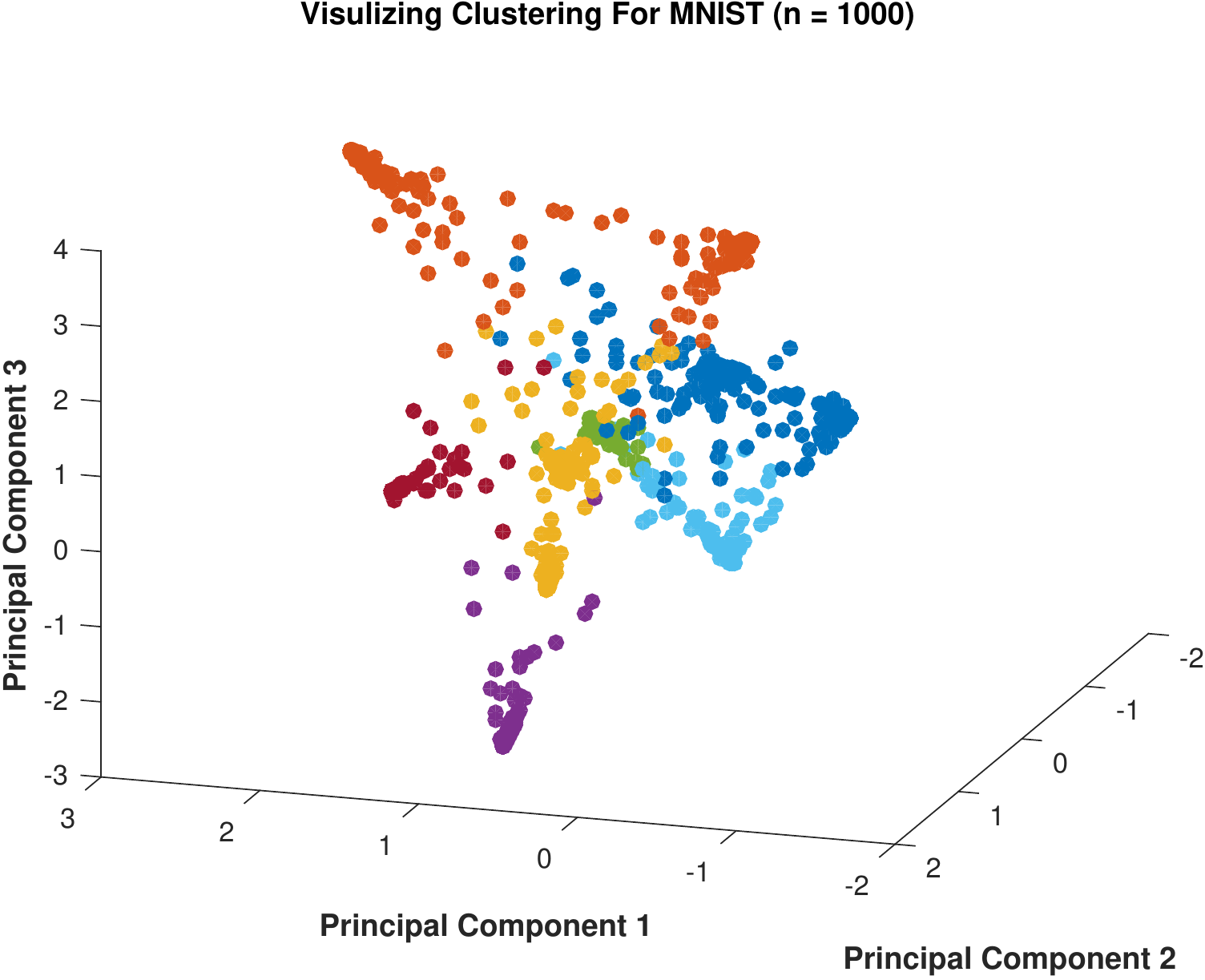}
		\caption{Selected recovery results by model (\ref{Eq: Modified_Model}) with $2$-norm. Left: Two Half-Moon data with $n = 1000$, $k=20$,  $\gamma = 5$. Middle: Unbalanced Gaussian data with $n = 6500$, $k=10$,  $\gamma = 1$. Right: a subset of MNIST with $n = 1000$,
			$\gamma = 1$.} \label{fig: visualization_sim}
	\end{center}
\end{figure*}

\subsubsection*{Two Half-Moon data}

The simulated data of two interlocking half-moons in $\R^2$ is one of the most popular
test examples in clustering.
Here we compare the computational time between our proposed {\sc Ssnal} and {\sc AMA} on this dataset with different problem scales. We note that AMA could not satisfy the stopping criteria (\ref{eq:stop_ama}) within $100000$ iterations when $n$ is large. 
In the experiments, we choose $k = 10$, $\phi = 0.5$ (for the weights $w_{ij}$)
and $\gamma \in [0.2:0.2:10]$ (in {\sc Matlab} notation) to generate the clustering path. After generating the clustering path with {\sc Ssnal}, we repeat the experiments using the same pre-stored
primal objective values and stop the AMA using
the criterion (\ref{eq:stop_ama}).
We report the average time for solving each problem (50 in total) in Table \ref{tab: cmp_time_halfmoon}.
Observe that our {\sc Ssnal} can be more than 50 times faster than AMA.

We also compare the recovery performance between the convex clustering model (\ref{Eq: Modified_Model}) and K-means (\ref{eq: K-means}). We choose the Rand Index \citep{hubert1985comparing} as the metric to evaluate the performance of these two clustering algorithms. In Figure \ref{fig:randindex-halfmoon}, we can see
that comparing to the K-means model, the convex clustering model is able to achieve  a much
better Rand Index, even when the number of clusters is not correctly identified.
\begin{table}[!h]
	\caption{Computation time (in seconds) comparison on the Two Half-Moon data. (--- means that the maximum number of 100,000 iterations is reached)}
	\label{tab: cmp_time_halfmoon}
	\begin{center}
	\begin{small}
		\vskip-5mm
			\begin{tabular}{ccccccc}
				\toprule
				$n$ & 200 & 500 &  1000 & 2000 & 5000& 10000\\
				\midrule
				AMA & 0.41 & 4.43 & 28.27 & 78.36 &---& ---\\
				{\sc Ssnal} & $\mathbf{0.11}$ & $\mathbf{0.19}$ &$\mathbf{0.49}$ & $\mathbf{0.91}$& $\mathbf{3.82}$ &$\mathbf{9.15}$\\
				\bottomrule
			\end{tabular}
	\end{small}
	\end{center}
	\vskip -0.1in
\end{table}

\begin{figure}[!h]	
	\vskip 0.2in
	\begin{center}
		\includegraphics[width=0.4\textwidth]{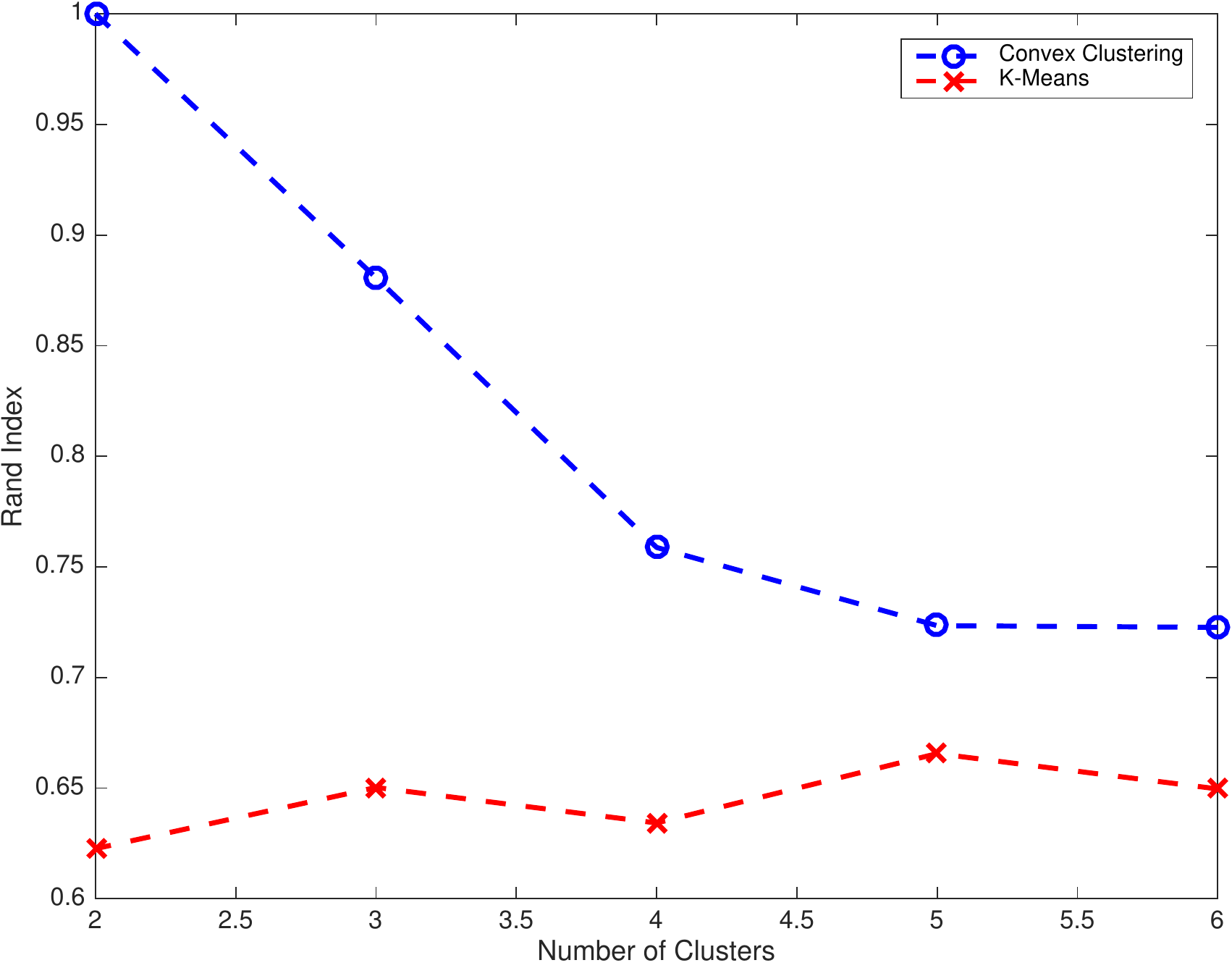}
		\quad
		\includegraphics[width=0.4\textwidth]{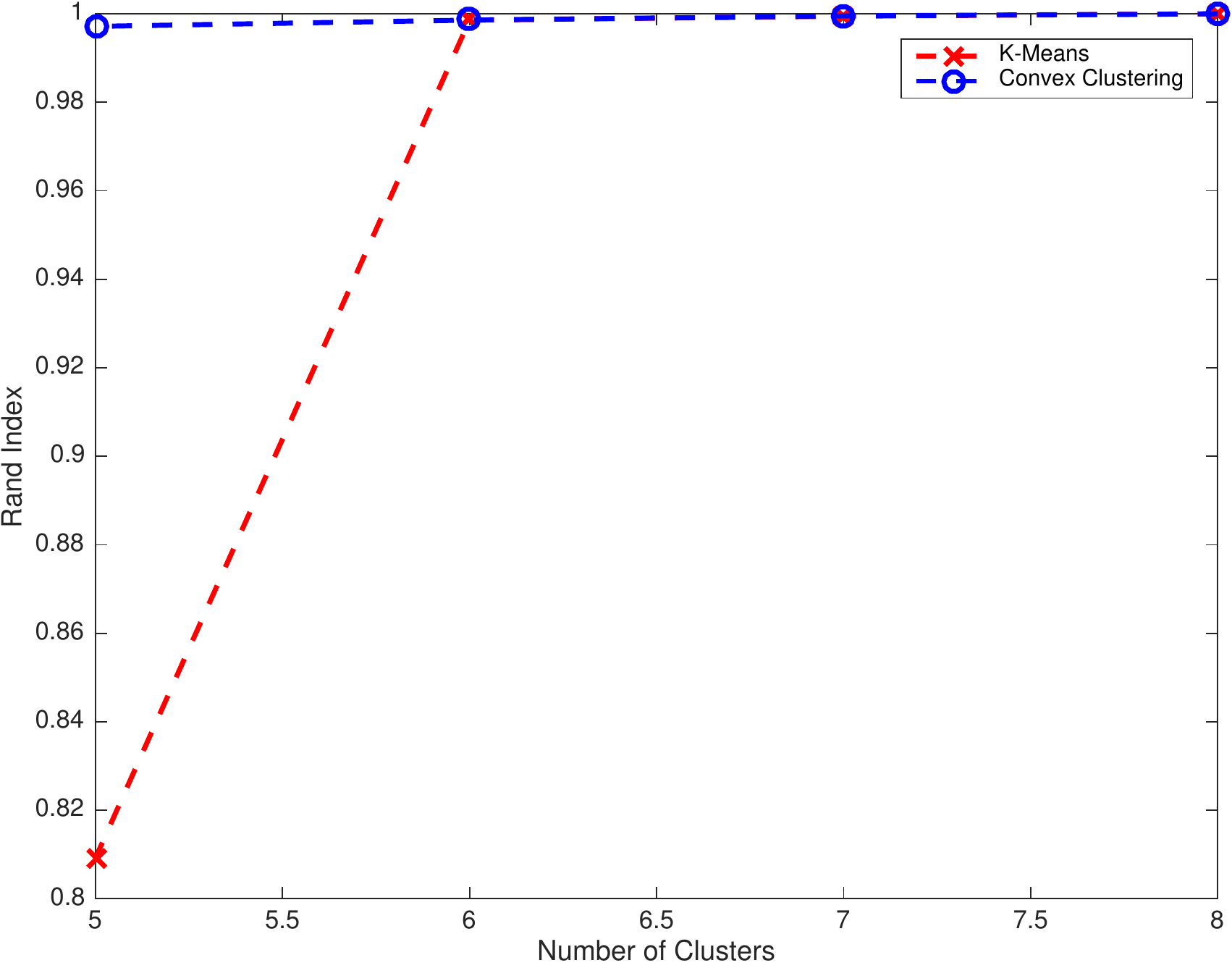}
	\end{center}
	\vskip -0.2in
	\caption{Clustering performance (in terms of the Rand Index) of the convex clustering and K-means models on the Two Half Moon dataset (left panel) and the Unbalanced Gaussian dataset (right panel).} \label{fig:randindex-halfmoon}
	\vskip -0.2in
\end{figure}

\subsubsection*{Unbalanced Gaussian and semi-spherical shells data}

Next, we show the performance of {\sc Ssnal} and AMA on the Unbalanced Gaussian data points in $\R^2$ \citep{UnbalanceSet}.
 In this experiment, we solve (\ref{Eq: Modified_Model}) with $k=10$, $\phi = 0.5$ and $\gamma \in [0.2:0.2:2]$.
For this dataset,
we have scaled it so that each entry is in the interval $[0, 1]$.
We can see from  Figure \ref{fig: visualization_sim}
 that the convex clustering model (\ref{Eq: Modified_Model}) can recover the cluster assignments perfectly with well chosen parameters.

In the experiments, we find that AMA has difficulties in reaching the stopping criterion (\ref{eq:stop_ama}).
We summarize some selected results in Table \ref{tab: result_unbalanced_gauss},
wherein we report the computation times and iteration counts for both AMA and {\sc Ssncg}.
Note that we report the number of {\sc Ssncg} iterations because each of these iterations constitute
the main cost for {\sc Ssnal}.
In
Figure \ref{fig:randindex-halfmoon},
we show the recovery performance between the convex clustering model and K-means on this dataset.

\begin{table}[!h]
	\caption{Numerical results on Unbalanced Gaussian data. }
	\label{tab: result_unbalanced_gauss}
	\vskip3mm	
	\begin{center}
	\begin{small}
		\begin{tabular}{llllll}
			\toprule
			$\gamma$ & 0.2 & 0.4 & 0.6 & 0.8 & 1.0\\
			\midrule
			$t_{\rm AMA}$& 264.54 &256.21 & 260.06 &262.16&263.27\\
			$t_{\mbox{\sc Ssnal}}$ & $\mathbf{1.15}$ & $\mathbf{0.57}$ &$\mathbf{0.65}$& $\mathbf{0.64}$& $\mathbf{0.83}$\\
			Iter$_{\rm AMA}$ & 100000 & 97560&97333&100000&100000\\
			Iter$_{\mbox{\sc Ssncg}}$&  $\mathbf{23}$ & $\mathbf{21}$ & $\mathbf{24}$ & $\mathbf{24}$ & $\mathbf{27}$\\
			\bottomrule
		\end{tabular}
	\end{small}
	\end{center}
	\vskip -0.1in
\end{table}


In order to test the performance of our {\sc Ssnal} on large data set, we also generate
a data set with $\mathbf{200,000}$ points in $\R^3$ such that 50\% of the points are uniformly distributed
in a semi-spherical shell whose inner and outer surfaces have radii equal to 1.0 and 1.4, respectively.
The other 50\% of the points are uniformly distributed in a concentric semi-spherical
shell whose  inner and outer surfaces have radii equal to 1.6 and 2.0, respectively.
Figure \ref{fig-spherical-shells} depicts the recovery result when we use only 6,000 points.
For the data set with $\mathbf{n=200,000}$, our algorithm takes only {\bf 374 seconds}  to solve
the model (\ref{Eq: Modified_Model}) when we choose $\gamma = 50$, $\phi=0.5$
and $k=10$.
In solving the problem, our algorithm used 32 {\sc Ssncg} iterations
and the average number of CG steps needed to solve the large linear
system (\ref{eq: newton-linearsystem}) is 79.3 only.
Thus, we can see that our algorithm can be very efficient in solving the
convex clustering model  (\ref{Eq: Modified_Model}) even when the data set is large.
Note that we did not run AMA as it will take too much time to solve the problem.

\begin{figure}[!h]	
	\begin{center}
		\includegraphics[width=0.9\textwidth]{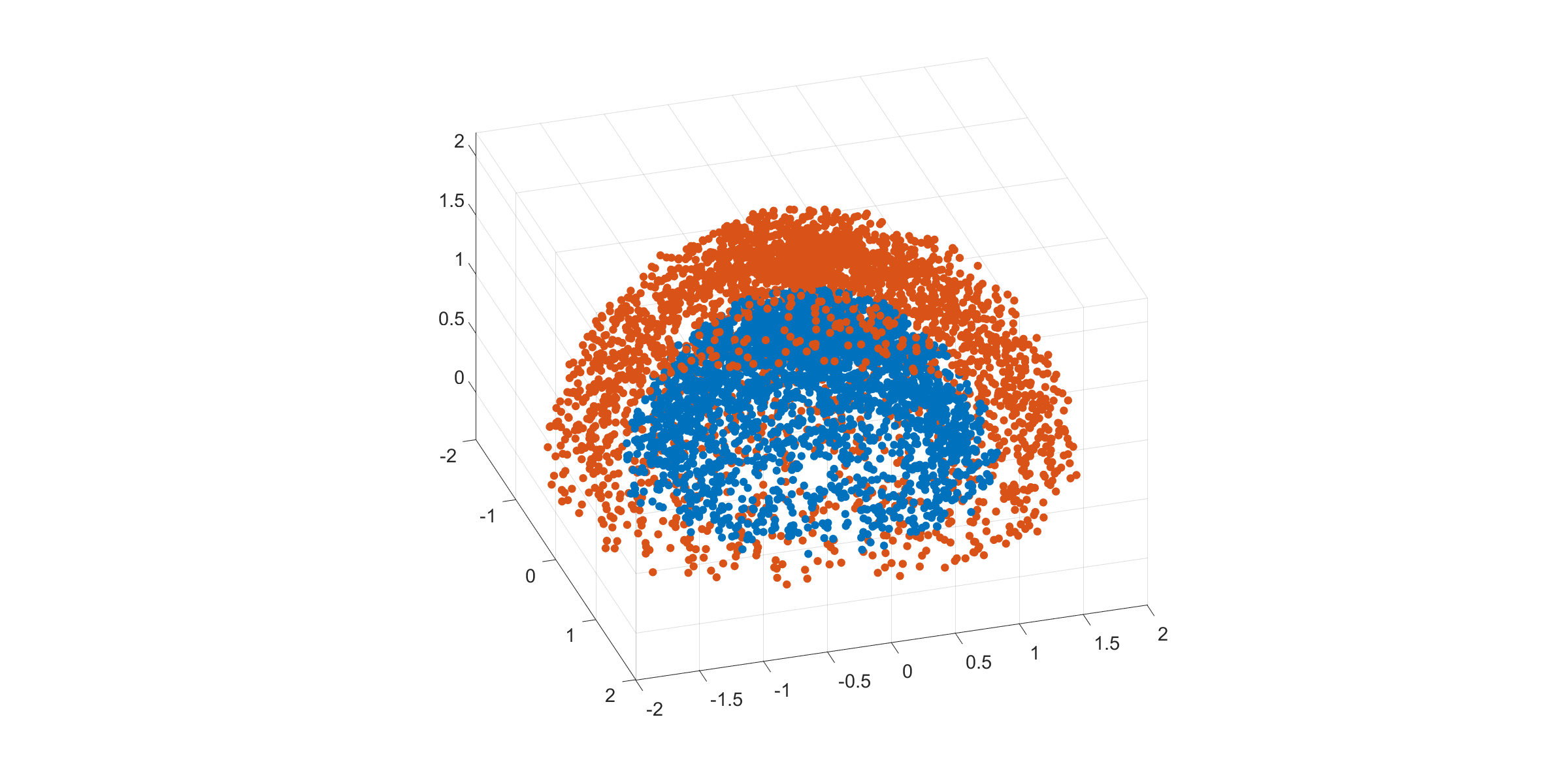}
	\end{center}
	\vskip -0.3in
	\caption{Recovery result by model  (\ref{Eq: Modified_Model}) for a
	semi-spherical shells data set with $6,000$ points.} \label{fig-spherical-shells}
	\vskip -0.2in
\end{figure}

\subsection{Real data}

In this section, we compare the performance of our proposed {\sc Ssnal} with AMA on some real datasets, namely, MNIST, Fisher Iris, WINE, Yale Face B(10Train subset).
For real datasets, a preprocessing step is sometimes necessary to transform the data to one
whose features are meaningful for clustering.
Thus, for a subset of MNIST (we selected a subset because AMA cannot handle the whole dataset), we first apply the preprocessing method described in \citep{mixon2016clustering}. Then we apply the
model (\ref{Eq: Modified_Model}) on the preprocessed data.
The comparison results between {\sc Ssnal} and {\sc AMA} on the real datasets are presented in Table
\ref{tab: cmp_time_real}. One  can observe that {\sc Ssnal} can be much more efficient than
{\sc AMA}.

\begin{table}[!h]
	\caption{Computation time comparison on real data. (*) means that the maximum of 100000 iterations is reached for
		all instances.}
	\label{tab: cmp_time_real}
	\begin{center}
	\begin{small}
	\vskip-5mm
			\begin{tabular}{lllll}
				\toprule
				Dataset & $d$ & $n$ &  AMA(s)& {\sc Ssnal}(s)\\
				\midrule
				MNIST & 10 & 1,000 & 79.48 & $\mathbf{1.47}$\\
				MNIST & 10 & 10,000 & 1753.8$^*$ & $\mathbf{69.3}$ \\
				Fisher Iris & 4 & 150 &0.58 & $\mathbf{0.16}$\\
				WINE & 13 & 178 &2.62 & $\mathbf{0.19}$\\
				Yale Face B& 1024 & 760 & 211.36 & $\mathbf{35.13}$\\
				\bottomrule
			\end{tabular}		
	\end{small}
	\end{center}
	\vskip -0.1in
\end{table}

\subsection{Sensitivity with different $\gamma$}
In order to generate a clustering path for a given dataset, we need to solve (\ref{Eq: Modified_Model}) for
a sequence of $\gamma > 0$. So the stability of the performance of the optimization algorithm with different $\gamma$ is very important. In our experiments, we have found that the performance of AMA is rather sensitive to the
value of $\gamma$ in that the time taken to solve problems with different values of
$\gamma$ can vary widely. However, {\sc Ssnal} is much more stable.
In  Figure 6,
we show the comparison between {\sc Ssnal} and AMA on both the
Two Half-Moon and MNIST datasets with $\gamma \in [0.2:0.2:10]$.
\begin{figure}[!h]
	\vskip 0.2in
	\begin{center}
			\includegraphics[width=0.4\textwidth]{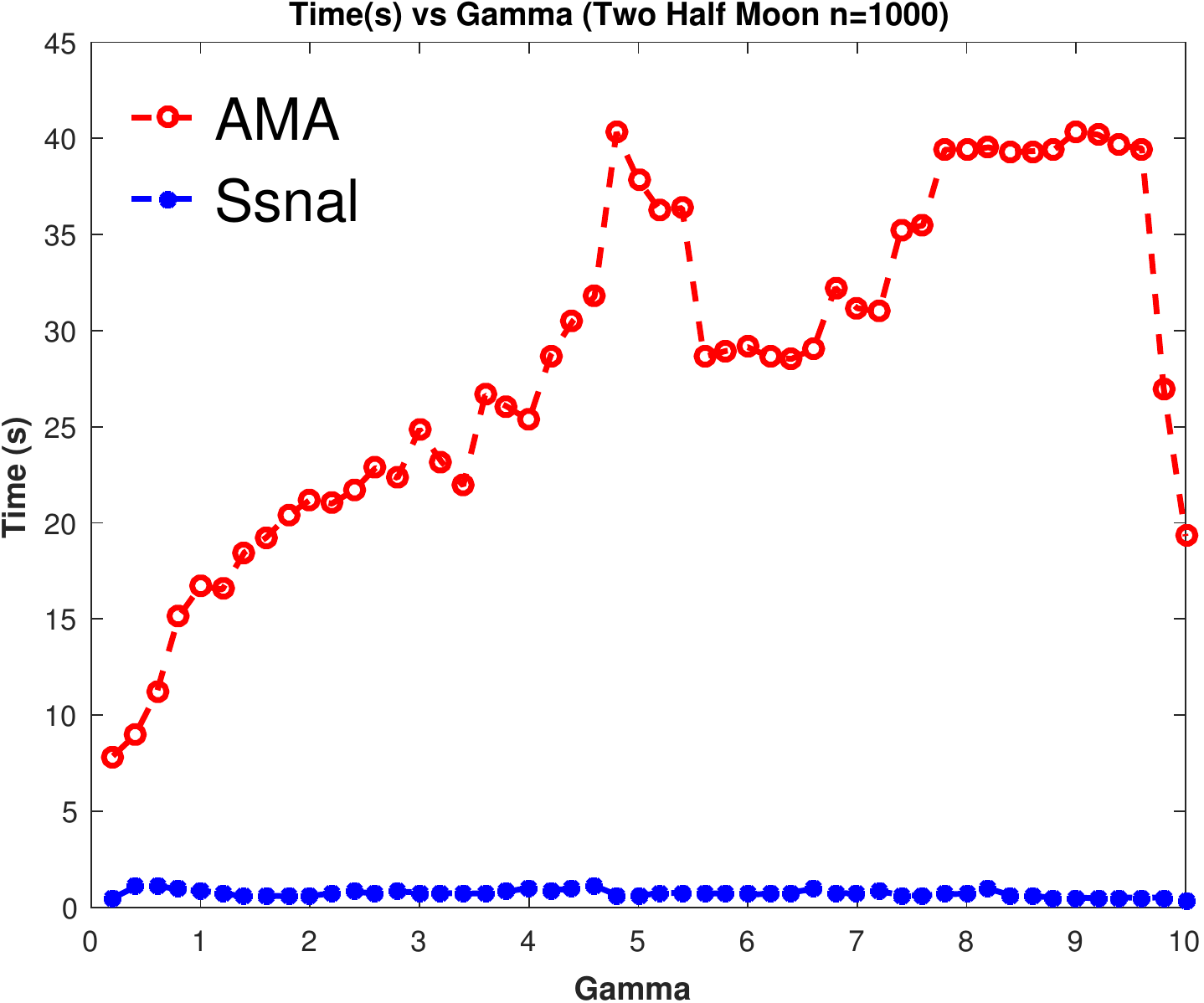}
			\includegraphics[width=0.4\textwidth]{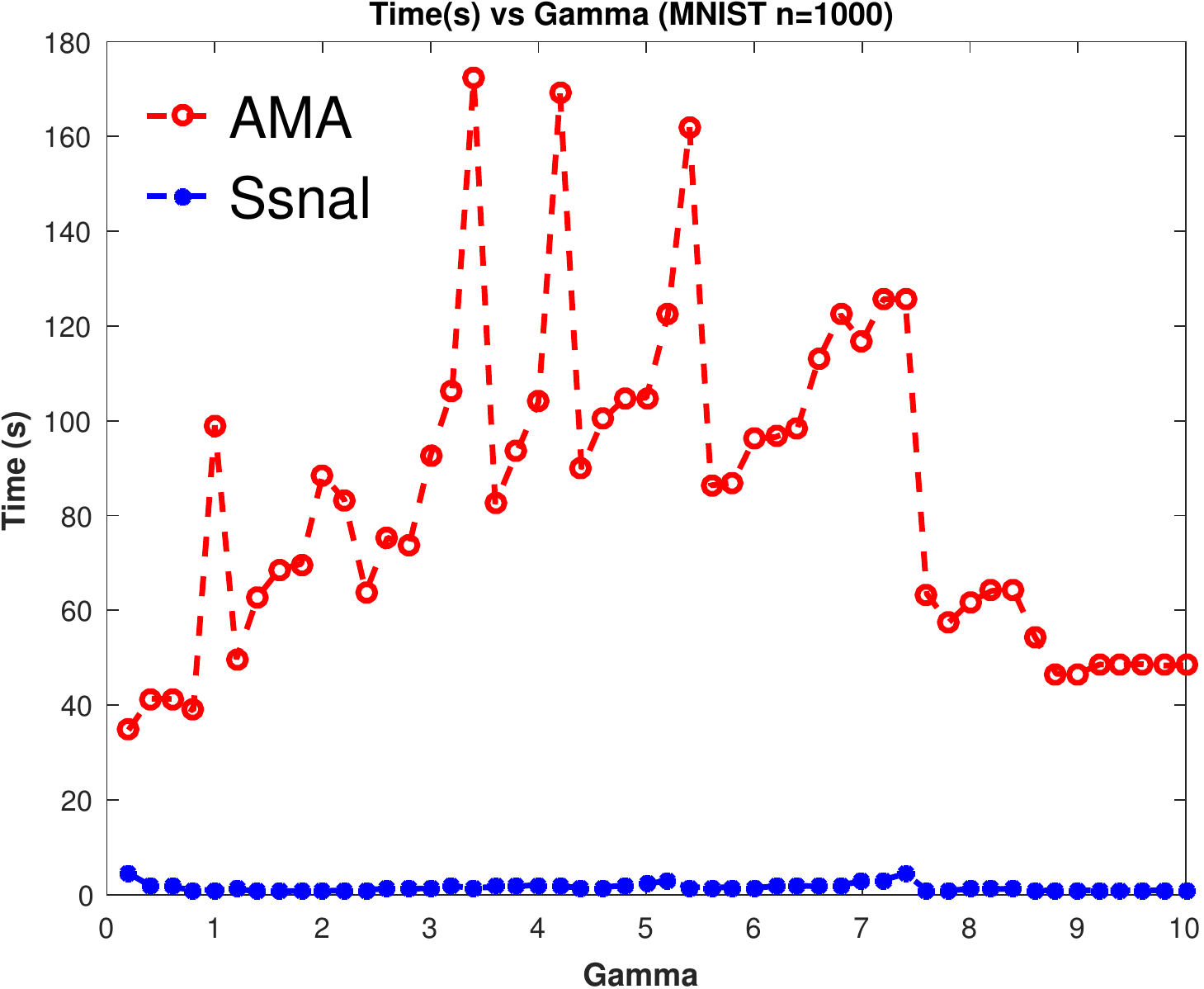}
		\caption{Time comparison between {\sc Ssnal} and {\sc AMA} on both Two Half-Moon and MNIST data with $\gamma \in [0.2:0.2:10]$.}
	\end{center}
	\vskip -0.2in
\end{figure}

\subsection{Scalability of our proposed algorithm}
In this section, we demonstrate the scalability of our algorithm {\sc Ssnal}. Before we show the numerical results, we give some insights as to why our algorithm could be scalable. Recall that
the most computationally expensive step in our framework is in using the semismooth Newton-CG method to solve (\ref{eq: nonsmooth_eq}). However, if we look inside the algorithm, we can see that the key step is to use the CG method to solve (\ref{eq: newton-linearsystem}) efficiently to get the Newton direction. According to our complexity analysis in Section \ref{sec: CG}, the computational cost for one step of  the CG method is $O(d|{\mathcal{E}}|+d|\widehat{\mathcal{E}}|)$.
By the specific choice of $\mathcal{E}$,
$|\cal E|$ and $|\widehat{\mathcal{E}}|$ should only grow
slowly with  $n$.
The low computational cost for the matrix-vector product in our CG method, the rapid
convergence of the CG method, and the
fast convergence of the {\sc Ssncg} are
the key reasons behind why our algorithm can be scalable and efficient.

In our experiments, we set $\phi = 0.5$, $k = 10$ (the number of nearest neighbors). Then we solve (\ref{Eq: Modified_Model}) with $\gamma \in [0.4:0.4:20]$. After generating the clustering path, we compute the average time for solving a single instance of (\ref{Eq: Modified_Model}) for each problem scale.
Another factor related to the scalability is the number of neighbors $k$ used
to generate $\mathcal{E}$ in (\ref{Eq: Modified_Model}). So, we also show the performance of {\sc Ssnal} with different values of $k$. For each $k \in [5:5:50]$, we generate the clustering path for the Two Half-Moon data with $n = 2000$. Then we report the average time for solving
a single instance of (\ref{Eq: Modified_Model})  for each $k$.
We summarize our numerical results in Figure 7.
We can observe that the computation time grows almost linearly with $n$ and $k$.
\begin{figure}[!h]\label{fig: scale}
	\vskip 0.2in
	\begin{center}
			\includegraphics[width=0.40\textwidth]{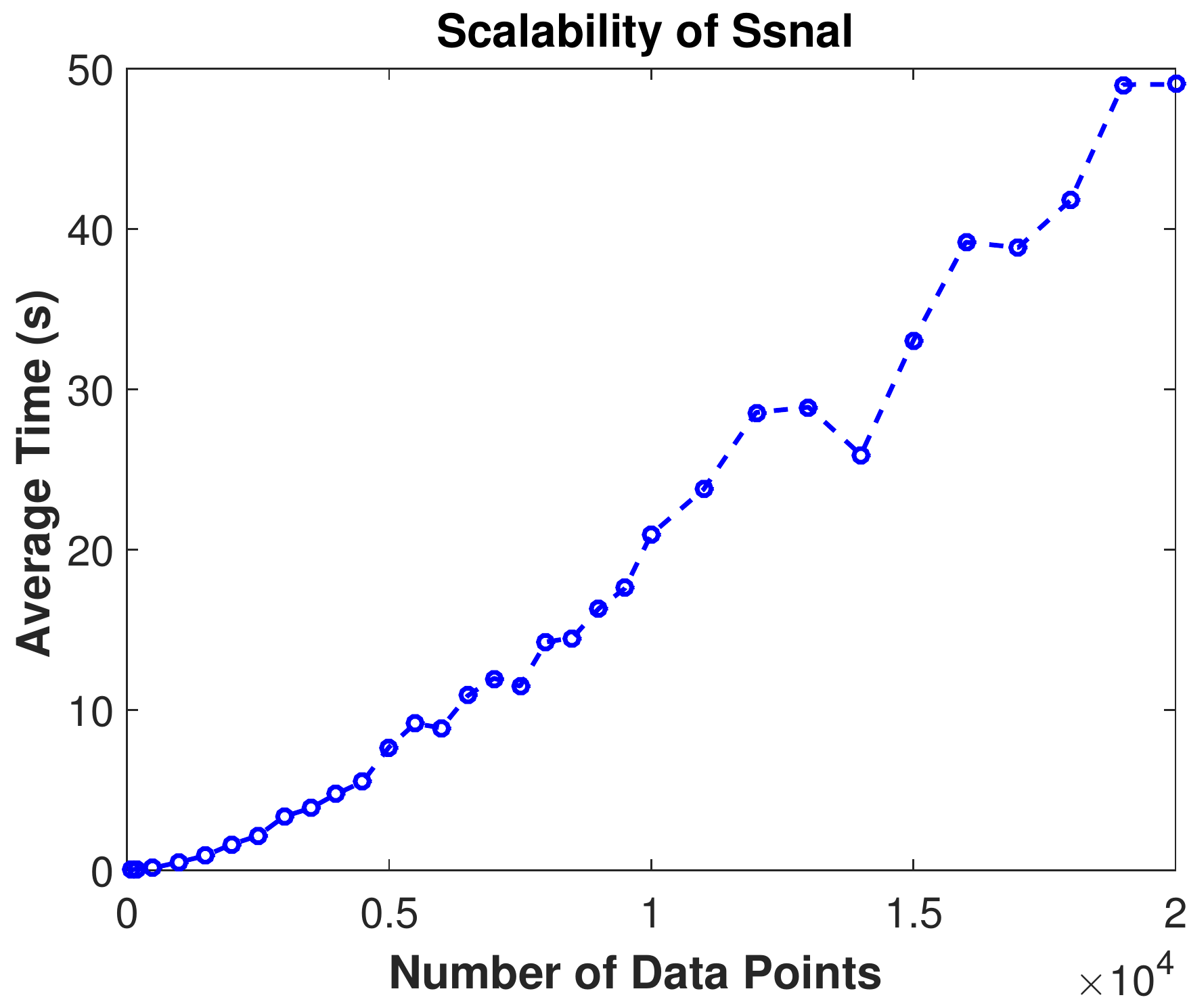}
			\includegraphics[width=0.40\textwidth]{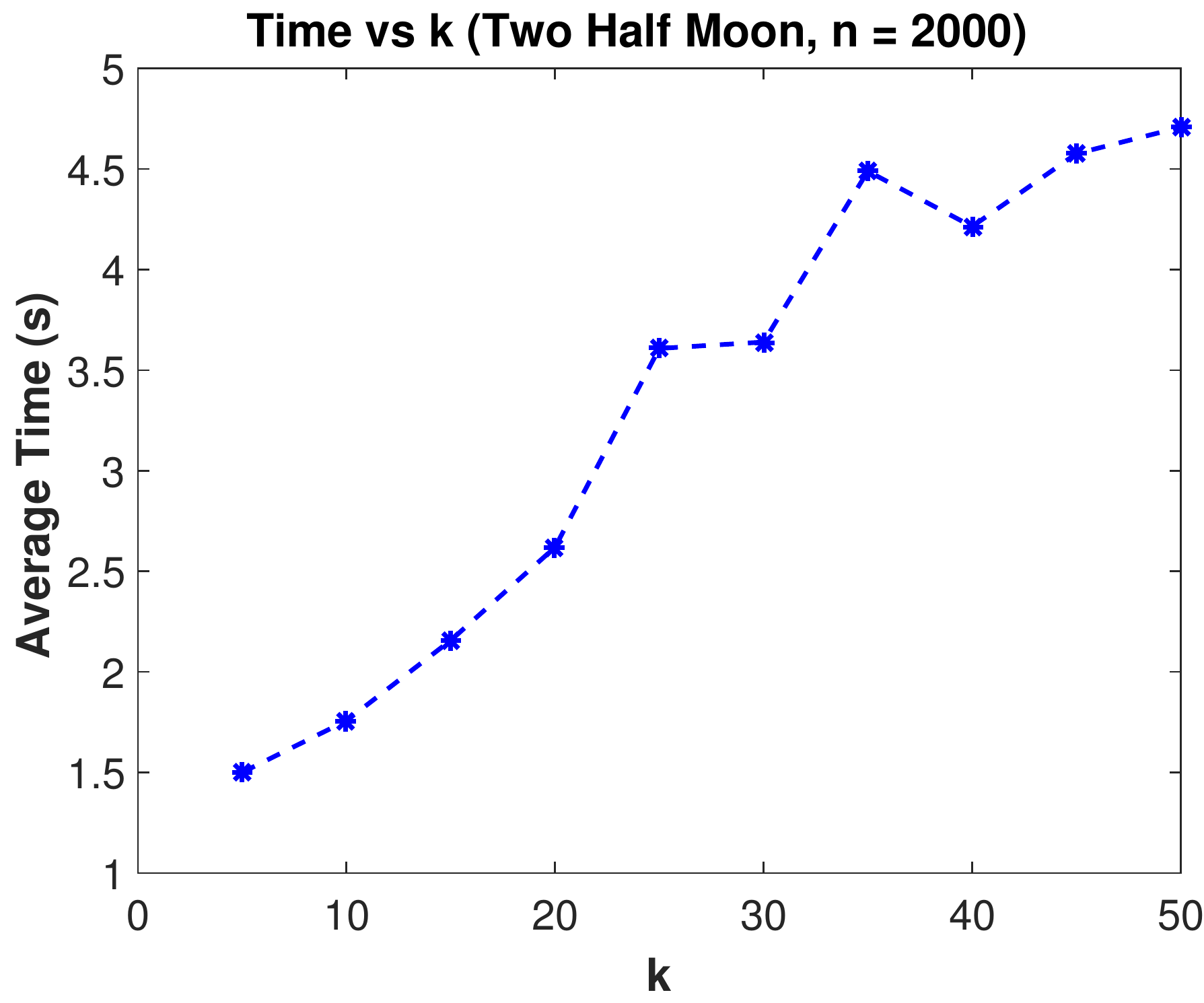}
		\caption{Numerical results to demonstrate the scalability of our proposed algorithm {\sc Ssnal} with respect to
			$n$ and $k$.}
	\end{center}
	\vskip -0.2in
\end{figure}

Comparing to the numerical  results reported in \cite{Chi15} and \cite{pmlr-v70-panahi17a} with $n \leq 500$ and $n \leq 600$, respectively, in our experiments, we apply our algorithm on the Half-Moon data with $n$ ranging
from $100$ to $20000$. Together with the semi-spherical shells with $\mathbf{200,000}$ data points, our results have convincingly demonstrated  the scalability of {\sc Ssnal}.
\section{Conclusion}
In this paper, we established the theoretical recovery guarantee for the general weighted convex clustering model, which includes many popular setting as special cases. The theoretical results we obtained serve to provide a more solid foundation for the convex clustering model. We have also proposed a highly efficient and scalable semismooth Newton based augmented Lagrangian method to solve the convex clustering model (\ref{Eq: Modified_Model}). To the best of our knowledge, this is the first optimization algorithm for convex clustering model which uses the second-order generalized Hessian
information. Extensive numerical results shown in the paper have demonstrated the scalability and superior performance of our proposed algorithm {\sc Ssnal} comparing to the
state-of-the-art first-order methods such as AMA and ADMM.
The convergence results for our algorithm are also provided.

As a possible
future work, we plan to design a distributed and parallel version of {\sc Ssnal}
with the aim to handle huge scale datasets. From the  modeling perspective, we will also work on generalizing our algorithm to handle kernel based convex clustering models.



%
%
%
%
%
%
%

\vskip 0.2in

\end{document}